\documentclass[11pt]{article}
\usepackage{fullpage}

\usepackage[utf8]{inputenc} 
\usepackage[T1]{fontenc}    
\usepackage{url}            
\usepackage{booktabs}       
\usepackage{amsfonts}       
\usepackage{nicefrac}       
\usepackage{multirow}
\usepackage{makecell}
\usepackage{hhline}
\usepackage{natbib}
\usepackage{mathtools}

\usepackage{tocloft}
\usepackage{braket}
\usepackage[a4paper,margin=1in]{geometry}
\usepackage[toc,page,header]{appendix}
\usepackage{minitoc}
\usepackage{booktabs}
\usepackage{multirow}

\usepackage[colorlinks,citecolor=blue,urlcolor=blue,linkcolor=blue,linktocpage=true,pagebackref=true]{hyperref}
\renewcommand*{\backref}[1]{}
\renewcommand*{\backrefalt}[4]{(\small%
    \ifcase #1 not cited%
          \or cited on page~#2%
          \else cited on pages #2%
    \fi%
    )}
\usepackage{wustyle}

\usepackage[most]{tcolorbox}

\newtcolorbox{highlightedtheorem}{%
  enhanced,
  breakable,
  colback=black!2,
  colframe=black!72,
  boxrule=0.65pt,
  arc=2pt,
  outer arc=2pt,
  left=3pt,
  right=3pt,
  top=3pt,
  bottom=3pt,
  before skip=10pt,
  after skip=10pt
}

\usepackage{wrapfig}

\usepackage{framed}
\colorlet{shadecolor}{orange!15}

\definecolor{c4}{RGB}{255,225,187}
\definecolor{c2}{RGB}{209, 233, 184}
\definecolor{c3}{RGB}{218,243,246}
\definecolor{c1}{RGB}{249, 229, 229}
\definecolor{c5}{RGB}{255, 128, 128}
\definecolor{c6}{RGB}{251, 132, 002}

\title{Optimal Learning Rate Schedules under Functional Scaling Laws: Power Decay and Warmup--Stable--Decay}

\author{
    Binghui Li$^{1,}$\thanks{Equal contribution. Emails: \texttt{libinghui@pku.edu.cn},\quad  \texttt{wangzilin@stu.pku.edu.cn}}\quad
    Zilin Wang$^{2,}$\footnotemark[1]\quad
    Fengling Chen$^{2}$\quad
    Shiyang Zhao$^{2}$ \\[-.1em]
    Ruiheng Zheng$^{2}$ \quad
    Lei Wu$^{1,2,3,}$\thanks{Corresponding author. Email: \texttt{leiwu@math.pku.edu.cn}}
    \\[.5em]
  $^1$Center for Machine Learning Research, Peking University
    \\[-.1em]
  $^2$School of Mathematical Sciences, Peking University 
    \\[-.1em]
  $^3$AI for Science Institute, Beijing
}
\date{(Accepted at COLT 2026)}

\begin{document}

\maketitle

\doparttoc
\faketableofcontents

\vspace{-1em}
\begin{abstract}
We study optimal learning rate (LR) schedules under the functional scaling law (FSL) framework~\citep{li2025functional}, which decomposes training dynamics into signal learning and noise forgetting. In power-law kernel regression, these two components are governed by a source exponent $s>0$ and a capacity exponent $q>1$, respectively, with smaller $s$ corresponding to harder tasks. For a fixed training horizon $N$, we characterize the schedules that minimize the final-step loss under a stability constraint and reveal a sharp phase transition. In the easy-task regime $s>1-1/q$, the optimal schedule follows power decay from the beginning of training; in the hard-task regime $s<1-1/q$, it becomes warmup--stable--decay (WSD)-like~\citep{hu2024minicpm}, staying at the largest admissible LR for most of training before a final decay. In both regimes, the decay exponent is $2q-1$: \textit{task difficulty determines when to decay, while model capacity determines how to decay}. Beyond the exact optimum, we study fractional schedules, whose shape is defined over relative training progress. We show that precise tuning of the decay shape is often unnecessary: a broad class of profiles attains the optimal convergence rate, while overly slow terminal decay leads to schedule-induced capacity saturation. Finally, for one-pass SGD in kernel regression, FSL-motivated power-decay schedules achieve optimal last-iterate rates. Experiments support the theoretical predictions and the task-dependent transition between early and delayed decay.
\end{abstract}

\section{Introduction}

Learning rate (LR) schedules are central to both the theory and practice of modern machine learning. In stochastic training, a large LR can accelerate training, but it also amplifies the injection of gradient noise. Conversely, a small LR suppresses such noise but may  slow down training progress. Thus, designing an effective LR schedule requires balancing these two competing effects over the course of training.

The study of LR schedules dates back to the seminal work of \citet{robbins1951stochastic}, which established classical convergence conditions for decaying step sizes in stochastic approximation. Subsequent theory has largely focused on polynomial-decay schedules $\eta_k\propto k^{-a}$ with $a\in(1/2,1)$ in convex and non-convex stochastic optimization~\citep{lacoste2012simpler,bubeck2014theory,shamir2013stochastic}. In linear regression, polynomial decay with iterate averaging~\citep{ruppert1988efficient} achieves minimax-optimal rates~\citep{bach2013non,dieuleveut2015non,mucke2019beating}, while more recent work has shown that exponential-decay schedules can attain nearly optimal last-iterate rates without averaging~\citep{ge2019step,wu2022last}.

\begin{wrapfigure}{r}{0.35\textwidth}
\centering
{\footnotesize Common LR schedules}\\
\includegraphics[width=\linewidth]{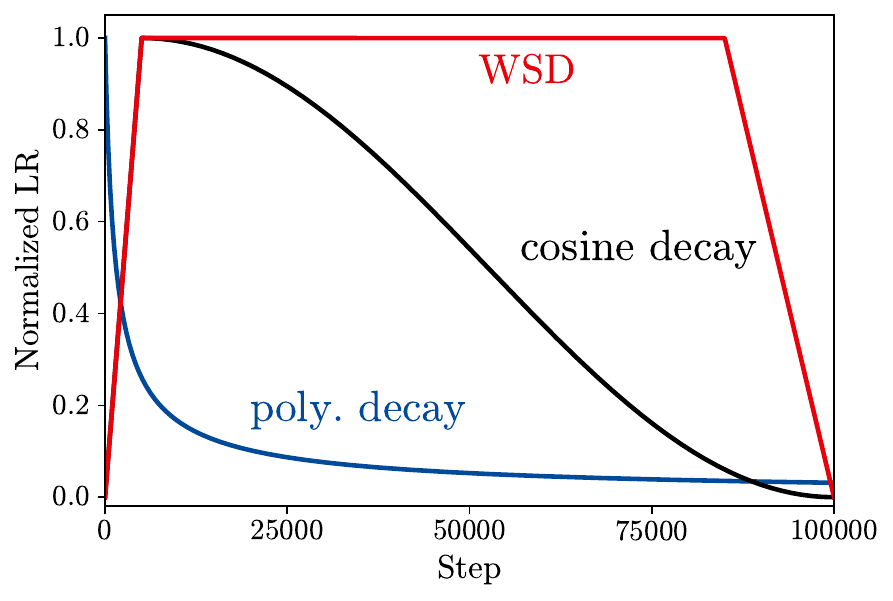}
\vspace{-2em}
\end{wrapfigure}
Despite the rich theoretical literature, our understanding of LR scheduling in practical training remains rather incomplete. Modern large-scale training often uses schedules quite different from those traditionally studied in theory, 
most notably \textit{cosine decay}~\citep{loshchilov2016sgdr,hoffmann2022training,touvron2023llama} and \textit{warmup--stable--decay} (WSD)~\citep{zhai2022scaling,hu2024minicpm,liu2024deepseek,team2025kimi}. 
In particular, WSD schedules keep the learning rate constant for most of training, for up to 80\% of the training horizon, and decay only in a short final annealing phase. The above figure illustrates these commonly used schedules. Despite this delayed decay, empirical studies have reported that WSD can match, and in some cases outperform, cosine decay in large-scale training~\citep{hu2024minicpm,hagele2024scaling,schaipp2025surprising}. This raises a basic question: \emph{what determines whether LR decay should begin early or be delayed until late in training?}

Bridging this theory--practice gap requires moving beyond the prevailing ``propose-and-analyze'' paradigm, which treats the schedule as a fixed input. A more systematic approach is to characterize an explicit map from the LR schedule to the resulting loss dynamics,
\[
\text{LR schedule }\eta \longmapsto \text{loss dynamics }L=\mathcal M[\eta].
\]
Such a map enables inverse design: optimizing the schedule under a given training budget.

Recent work by \citet{li2025functional} provides such a map through the \textbf{functional scaling law (FSL)}. Let $t_k=\sum_{i=0}^{k-1}\eta_i$ denote the \emph{intrinsic time} and let $\cE_k$ denote the expected risk at step $k$. FSL describes the risk dynamics, up to constant factors, as
\begin{equation}\label{eqn:intro-fsl}
\cE_k\eqsim \cS(t_k)+\sum_{j=0}^{k-1}\cK(t_k-t_j)\eta_j^2,
\end{equation}
where $\cS$ describes signal learning and $\cK$ is the forgetting kernel. In power-law kernel regression, $\cS(t)=(c_S+t)^{-s}$ and $\cK(t)=(c_K+t)^{-(2-1/q)}$ for constants $c_S,c_K\eqsim1$. Here the source exponent $s>0$ characterizes task difficulty and the capacity exponent $q>1$ governs noise forgetting. Section~\ref{sec:learning_setup_and_functional_description} gives the formal setup and detailed interpretations of the two terms. The exponents admit a first-principles spectral interpretation in kernel regression, while the same FSL also  describes how LR schedules affect the loss dynamics in LLM pretraining~\citep{li2025functional}.

\begin{figure}[!t]
\centering 
\includegraphics[width=0.32\textwidth]{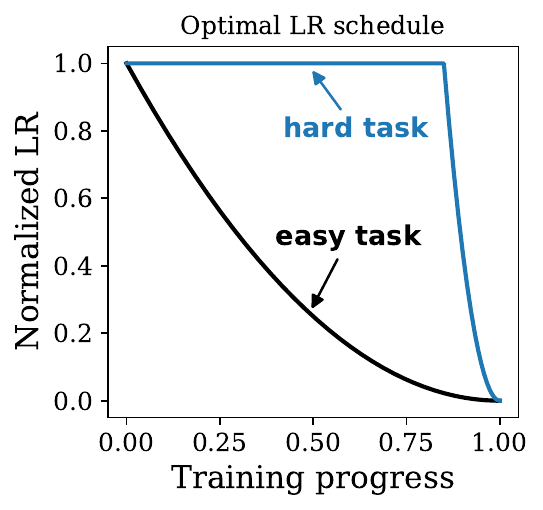}
\hspace*{-.3em}
\includegraphics[width=0.32\textwidth]{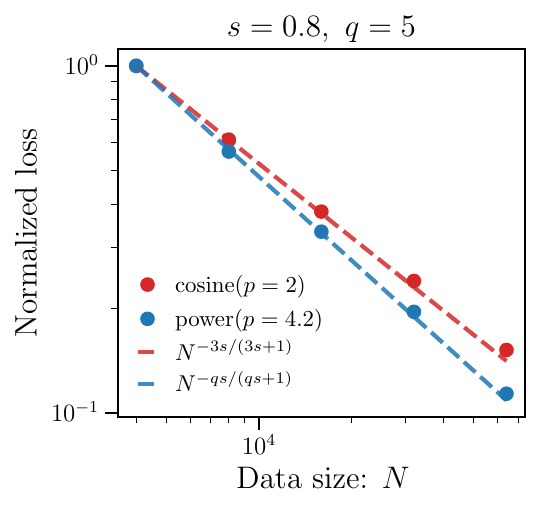}
\hspace*{.4em}
\raisebox{.7em}{\includegraphics[width=0.31\textwidth]{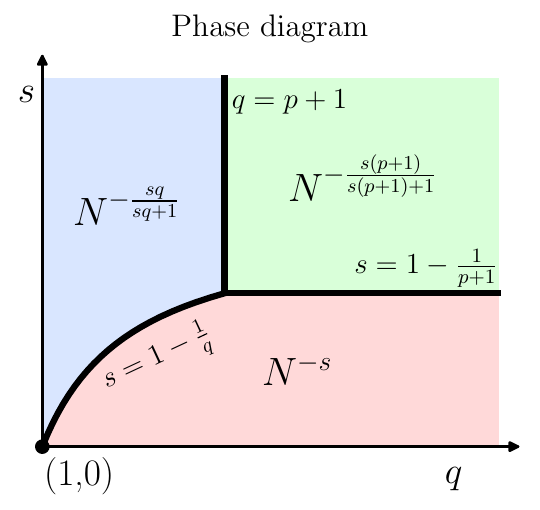}}
\caption{
\textbf{\small Summary of the main predictions in PLK regression.}
Here, $s$ controls task difficulty, $q$ model capacity, and $p$ the terminal decay of a fractional schedule.
\textbf{(left)} Easy tasks favor decay from the beginning, whereas hard tasks favor a WSD-like schedule with a long stable phase.
\textbf{(middle)} Schedule-induced capacity saturation: for $(s,q)=(0.8,5)$, sufficiently fast power decay attains the optimal scaling rate, while cosine decay follows the predicted saturated rate. Peak LRs are tuned separately for each $N$, and losses are averaged over $500$ SGD runs.
\textbf{(right)} Phase diagram for fractional schedules (Theorem~\ref{thm: optimal-fractional-lrs}). The boundary $q=p+1$ separates the capacity-limited and schedule-limited regimes. Blue denotes the minimax-optimal regime, red the one-pass-SGD-optimal regime, and green the rate-suboptimal regime.}
\label{fig: intro}
\end{figure}

\paragraph{Our contributions.}
We use FSL to solve the LR-schedule design problem under a fixed training horizon and stability constraint, defining the optimal schedule as the one that minimizes the final-step excess risk. Our main results are:

\begin{itemize}

\item \textbf{Optimal LR schedules exhibit a phase transition.}
Let $s_c=1-1/q$. In the easy-task regime $s>s_c$, the optimal schedule decays from the beginning of training and, for some $\delta_N=o_N(1)$, has the form
\[
\eta^*(z)=\eta_{\peak}\left(1+\delta_N-\frac{z}{N}\right)^{2q-1},
\qquad
\eta_{\peak}\eqsim N^{-\frac{s-1+1/q}{s+1/q}}.
\]
In the hard-task regime $s<s_c$, the optimal peak LR saturates at the stability limit and the schedule becomes WSD-like, staying at the largest admissible LR for most of training before following the same power-law decay near the end. Thus, task difficulty determines \emph{when} decay should begin, while model capacity determines \emph{how} the LR should decay. 

\item \textbf{Optimal rates do not require the exact optimal shape.}
We next consider fractional schedules $\eta_N(z)=\bar{\eta}\zeta(z/N)$, which separate the LR scale from its shape over relative training progress. We show that any profile with terminal decay $\zeta(x)\asymp(1-x)^p$ and $p>q-1$ attains the optimal convergence rate after optimally tuning $\bar{\eta}$. When $p<q-1$, terminal decay becomes the bottleneck and the rate no longer improves with increasing $q$, a phenomenon we call \emph{schedule-induced capacity saturation}. We further show that normalization by the training horizon is essential: any family $\eta_N(z)=\eta_0 h(z)$ with a fixed non-increasing profile $h$ in absolute training steps is rate-suboptimal, even when $\eta_0$ is retuned for each $N$. Classical polynomial decay is a canonical example.

\item \textbf{Optimal last-iterate rates for kernel regression.}
For one-pass SGD with FSL-motivated power-decay schedules, the last iterate attains the minimax-optimal rate without logarithmic factors in the easy-task regime and the optimal rate among one-pass SGD methods in the hard-task regime. To the best of our knowledge, this closes the logarithmic gap between last-iterate and averaged SGD in the easy-task regime for the first time~\citep{wu2022last,lin2024scaling,zhang2024optimality}.
\end{itemize}

Figure~\ref{fig: intro} summarizes the main predictions of our theory. Additional PLK regression experiments covering different schedule families and the three regimes in Figure~\ref{fig: intro}(right) are provided in Appendix~\ref{app:additional-experiments}. Experiments on neural-network training beyond the kernel setting show the same qualitative trend: harder tasks favor later LR decay.

\section{Related work}

\paragraph*{Neural scaling laws.}
\citet{hestness2017deep} first observed predictable relationships between model performance and scale, including model  and dataset size; these relationships were later formalized as neural \emph{scaling laws}~\citep{kaplan2020scaling}.
These laws have since guided large-scale training and been refined across architectures and training regimes~\citep{henighan2020scaling, hoffmann2022training, kadra2023power, aghajanyan2023scaling, muennighoff2023scaling, kumar2024scaling, tissue2024scaling, luo2024a}, with parallel theoretical efforts explaining their  origins and mechanisms~\citep{bordelon2024dynamical,lin2024scaling, bahri2024explaining, paquette2024fourplus,lin2025improved, yan2025larger,li2026muon}.
Our work fits into this line of research by providing a scaling-law analysis of LR schedule design. Specifically, we leverage the FSL introduced by~\citet{li2025functional} to characterize the structure of optimal LR schedules.  Closely related, \citet{wang2026fast} use FSL to study the design of batch size scheduling.

\paragraph{Optimal LR schedules for linear regression.} Previous work has studied LR-schedule design for linear regression through optimal control, but directly optimizing the full SGD dynamics is difficult and typically requires low-dimensional or simplified settings~\citep{li2017stochastic,fahrbach2023learning}. A notable exception is \citet{pan2021eigencurve}, who derive a Hessian-spectrum-dependent schedule for finite-dimensional stochastic quadratics and prove its rate optimality up to constant factors. Their analysis, however, verifies a prescribed schedule rather than optimizing over the full class of admissible schedules. FSL reduces the high-dimensional dynamics under power-law structure to an explicit functional of the LR schedule, making this inverse-design problem tractable and allowing us to characterize the optimal schedule itself.

\paragraph{Understanding cosine and WSD schedules.}
For cosine schedules, \citet{li2021second} attribute their empirical success to a \emph{duration-aware} property: the cumulative learning rate scales linearly with the total training horizon while the learning rate itself decays to a horizon-independent minimum. Standard polynomial decay schedules fail to satisfy this dual requirement. In this work, we generalize this property through a class of \emph{fractional} LR schedules and identify the regimes in which it attains optimal rates and those in which it does not.

For WSD schedules, \citet{wen2024understanding} provide a river-valley landscape interpretation of their dynamical behavior; \citet{schaipp2025surprising} and \citet{li2025functional} offer theoretical evidence for the benefits of delayed decay  but do not characterize the role of the \emph{decay shape} itself. Meanwhile, existing empirical studies reach seemingly conflicting conclusions: \citet{defazio2023optimal} and \citet{bergsma2025straight} advocate linear decay, whereas \citet{hagele2024scaling} report superior performance for concave decay profiles such as $1-\sqrt{x}$. In this work, we show that the optimal decay shape is determined by the profile of the forgetting kernel, which encodes how stochastic gradient noise is forgotten over training.
 In addition, \citet{luo2024a} numerically solve a variational problem based on a multipower-law model for LLM pretraining and find that the resulting schedule takes a WSD form with a decay shape approximately $(1-x)^{1.5}$. Our analysis provides theoretical support for this empirical observation.

\paragraph{One-pass SGD for kernel regression.}
The convergence of one-pass SGD for kernel regression, often formulated as high-dimensional feature-space linear regression, has received considerable attention. In particular, \citet{dieuleveut2015non} and \citet{mucke2019beating} showed that averaged SGD achieves the minimax-optimal rate $N^{-\frac{sq}{sq+1}}$ in the easy-task regime and the rate $N^{-s}$ in the hard-task regime. Subsequent work replaced iterate averaging by the more practical last iterate through suitable LR schedules, but typically incurred logarithmic factors~\citep{wu2022last,lin2024scaling,li2025functional}. We show that power decay, derived from our FSL-based schedule design, removes these logarithmic factors for the last iterate.

\vspace*{-.4em}
\section{Learning setup and functional description}
\label{sec:learning_setup_and_functional_description}

\indent\textbf{Notation.} We write $\eqsim$ for equivalence up to multiplicative constants, and $\lesssim$ and $\gtrsim$ for inequalities up to multiplicative constants. For two nonnegative functions $f,g:\mathbb{R}_{\ge0}\to\mathbb{R}_{\ge0}$, we write $f(t)\eqsim g(t)$ if there exist constants $C_1,C_2>0$, independent of $t$, such that $C_1f(t)\le g(t)\le C_2f(t)$ for all $t\ge0$. We write $a_N=o_N(1)$ if $\lim_{N\to\infty}a_N=0$.

\subsection{Power-law kernel regression}
\label{subsec: fslr}

Consider supervised learning with inputs $\bx\in\cX$ and labels $y\in\RR$. Let $\cD$ be a distribution over $\cX\times\RR$ and denote by $\cD_{\cX}$ its marginal on $\cX$. Samples are generated by $y=f^*(\bx)+\epsilon$, where $\epsilon\sim\cN(0,\sigma^2)$ is independent of $\bx$. 
Let $\bphi:\cX\to\ell^2$ be a feature map and define the associated kernel
$
k_{\phi}(\bx,\bx'):=\langle\bphi(\bx),\bphi(\bx')\rangle.
$
Equivalently, we work with kernel regression in its feature-space representation. The corresponding feature covariance operator is
\[
\bH=\EE_{\bx\sim\cD_{\cX}}\!\left[\bphi(\bx)\bphi(\bx)^\top\right].
\]
Let  $(\lambda_j)_{j\ge1}$ denote the eigenvalues of $\bH$  ordered non-increasingly. We assume that the target function admits the expansion $f^*(\bx)=\sum_{j=1}^\infty\theta_j^*\phi_j(\bx)$ in $L^2(\cD_{\cX})$, for some coefficient sequence $(\theta_j^*)_{j\ge1}$ satisfying $\sum_{j=1}^\infty\lambda_j|\theta_j^*|^2<\infty$.

Let $\{(\bx_k,y_k)\}_{k=0}^{N-1}$ be $N$ samples drawn independently from $\cD$. We train the linear model $f(\bx;\btheta)=\langle\btheta,\bphi(\bx)\rangle$ by \textbf{one-pass SGD}:
\begin{equation}
\label{equ: sgd_update}
\btheta_{k+1}=\btheta_k-\eta_k\nabla_{\btheta}\Big(\tfrac12(f(\bx_k;\btheta_k)-y_k)^2\Big),\qquad \btheta_0=\mathbf0,
\end{equation}
where $(\eta_0,\eta_1,\dots,\eta_{N-1})$ denotes the LR schedule. We measure performance by the excess risk
\[
\cE(\btheta)\coloneqq\EE_{\bx\sim\cD_{\cX}}\!\left[\tfrac12(f(\bx;\btheta)-f^*(\bx))^2\right].
\]

\begin{assumption}[Constant label noise]
$\sigma\eqsim 1$.
\end{assumption}

\begin{assumption}[Power-law structure]
\label{ass:power-law}
For some $q>1$ and $s>0$,
$
\lambda_j\eqsim j^{-q},  \lambda_j|\theta_j^*|^2\eqsim j^{-(1+sq)}.
$
\end{assumption}

\begin{assumption}[Hypercontractive features]
\label{ass: hypercontra}
Uniformly over $\bu,\bv\in\ell^2$,
\[
\EE\left[\langle \bu,\bphi(\bx)\rangle^2\langle \bv,\bphi(\bx)\rangle^2\right]\eqsim \EE\left[\langle \bu,\bphi(\bx)\rangle^2\right]\EE\left[\langle \bv,\bphi(\bx)\rangle^2\right],
\]
where the expectations are over $\bx\sim\cD_{\cX}$.
\end{assumption}

Assumption~\ref{ass:power-law} specifies the power-law source and capacity structure of the \emph{power-law kernel (PLK) regression}. The \emph{capacity exponent} $q$ controls the decay of the feature spectrum, with smaller $q$ corresponding to higher model capacity, while the \emph{source exponent} $s$ controls the spectral distribution of the target, with smaller $s$ corresponding to harder tasks. Such power-law assumptions are common in scaling-law analyses~\citep{bordelon2020spectrum,maloney2022solvable,lin2024scaling,li2025functional}. Section~\ref{sec:kernel-sgd} later replaces these  two-sided power laws by the standard source and capacity conditions used in kernel regression.

Assumption~\ref{ass: hypercontra} controls mixed fourth moments of feature projections by its second moments. It is satisfied, for example, by centered Gaussian features.

\vspace{-.5em}
\subsection{Functional scaling laws}
We use the functional scaling law (FSL) of \citet{li2025functional} as a macroscopic description of the training dynamics for the PLK regression problem above. Let
$
t_k\coloneqq\sum_{i=0}^{k-1}\eta_i
$
denote the intrinsic time accumulated up to step $k$. By \citet[Theorem~4.4]{li2025functional}, the  expected excess-risk dynamics of SGD~\eqref{equ: sgd_update} admit the decomposition
\begin{equation}\label{eqn:discrete_fsl}
\EE[\cE(\btheta_k)]\eqsim \cS(t_k)+\sum_{j=0}^{k-1}\cK(t_k-t_j)\eta_j^2.
\end{equation}
Here, $\cS$ is the \emph{signal-learning function}, which describes the noiseless learning dynamics, while $\cK$ is the \emph{forgetting kernel}, which describes how the effect of injected noise decays with subsequent intrinsic time.  Under Assumption~\ref{ass:power-law},
\begin{equation}
\cS(t)=(c_S+t)^{-s},\qquad \cK(t)=(c_K+t)^{-(2-1/q)}
\end{equation}
for constants $c_S,c_K\eqsim1$.

The FSL~\eqref{eqn:discrete_fsl} decomposes the loss dynamics into two terms:
\begin{itemize}
\item \textbf{Signal learning.} The signal-learning term $\cS(t_k)$ describes the reduction of error associated with learning the target $f^*$. Its decay is governed by the source exponent $s$: harder tasks have smaller $s$ and therefore slower signal learning.

\item \textbf{Noise injection and forgetting.} The second term tracks the effect of stochastic noise. Step $j$ injects noise at scale $\eta_j^2$, while the forgetting kernel $\cK(t_k-t_j)$ measures how much remains after the subsequent intrinsic time $t_k-t_j$. The forgetting rate is governed by the capacity exponent $q$: higher-capacity models have smaller $q$ and therefore forget noise more slowly.
\end{itemize}
For the analysis below, it is convenient to work on continuous training-step time. Let $\eta(z)$ denote the LR at step $z$ and define
\begin{equation}
t(z)\coloneqq\int_0^z\eta(u)\dd u,\qquad \varphi(\tau)\coloneqq\eta(t^{-1}(\tau)),
\end{equation}
where $t^{-1}$ denotes a generalized inverse when $t$ is not strictly increasing. Let $\bar{\cE}_t$ denote the expected loss at intrinsic time $t$. The continuous-time FSL then takes the form
\begin{equation}\label{eqn: plk_fsl}
\bar{\cE}_t\eqsim \cS(t)+\int_0^t \cK(t-\tau)\varphi(\tau)\dd\tau.
\end{equation}

\paragraph{Implications for schedule design.}
For a fixed training horizon $N$, let
$
T_N\coloneqq t(N)=\int_0^N\eta(z)\dd z
$
denote the total intrinsic time. Changing variables in \eqref{eqn: plk_fsl} gives the final-step FSL objective
\begin{equation}\label{eqn:fsl-training-steps}
\cF[\eta]\coloneqq \cS(T_N)+\int_0^N\cK\!\left(\int_z^N\eta(u)\dd u\right)\eta(z)^2\dd z.
\end{equation}
This expression makes the role of the LR schedule explicit. The total area $T_N$ controls signal learning; the LR $\eta(z)$ controls the noise injected at step $z$; and the remaining area $\int_z^N\eta(u)\dd u$ controls how much of that noise can be forgotten before training ends. Reducing the LR over an interval therefore suppresses fresh noise, but also reduces the intrinsic time accumulated during that interval for signal learning and for forgetting earlier noise.

For a fixed $T_N$, the signal term is constant, so schedule design reduces to allocating intrinsic time over training to minimize the noise term. This naturally separates the design problem into two questions: how much total intrinsic time to accumulate, and how to distribute it over training. These two choices are still coupled in the full optimization, since both the LR scale and its temporal profile affect noise injection and forgetting. 

\section{Optimal learning rate schedules}
\label{sec:opt-lrs}

We now solve the inverse LR-schedule design problem for a fixed training horizon $N$. Specifically, we minimize the final-step FSL objective~\eqref{eqn:fsl-training-steps} over LR schedules subject to the stability constraint:
\begin{equation}
\label{equ: functional_physical}
\min_{\eta}\;\cF[\eta]
\qquad \text{s.t.}\;\,
0\leq \eta(z)\leq \eta_{\stability}
\quad \text{for a.e. } z\in[0,N].
\end{equation}
Here, $\eta_{\stability}\eqsim 1$ is a fixed stability threshold, modeling the largest LR for stable training~\citep{wu2018sgd,wu2022alignment,wu2023implicit}. As we show below, whether this constraint becomes active leads to a sharp transition between immediate decay and delayed decay.

\begin{highlightedtheorem}
\begin{theorem}[Optimal LR schedules]
\label{thm: optimal_lrs}
Let $\eta^*$ be a minimizer of~\eqref{equ: functional_physical} and $\cE_N^*\coloneqq\cF[\eta^*]$. Then:
\begin{itemize}
\item \textbf{Easy-task regime} ($s>1-\frac{1}{q}$). There exists $\delta_N=o_N(1)$ such that the optimal schedule has the power-decay form
\[
\eta^*(z)=\eta_{\peak}\left(1+\delta_N-\frac{z}{N}\right)^{2q-1},
\qquad
\eta_{\peak}\eqsim N^{-\frac{s-1+1/q}{s+1/q}},
\]
and the optimal final-step loss satisfies
\begin{equation}\label{eqn: minimax-optimal-easy}
\cE_N^*\eqsim N^{-\frac{sq}{sq+1}}.
\end{equation}

\item \textbf{Hard-task regime} ($s<1-\frac{1}{q}$). The optimal schedule has a WSD-like form:
\[
\eta^*(z)=
\begin{cases}
\eta_{\stability}, & 0\le z\le N_1,\\[4pt]
\eta_{\stability}\left(1+\delta_N-\frac{z-N_1}{N-N_1}\right)^{2q-1}, & N_1<z\le N,
\end{cases}
\]
where $\delta_N=o_N(1)$ and the decay ratio satisfies
\begin{equation}\label{equ: decay_ratio}
r_N^*\coloneqq\frac{N-N_1}{N}\eqsim N^{-\frac{(1-1/q)-s}{2-1/q}}=o_N(1).
\end{equation}
Moreover,
\[
\cE_N^*\eqsim N^{-s}.
\]
\end{itemize}
\end{theorem}
\end{highlightedtheorem}

See Appendix~\ref{appendix: proof_opt_lrs} for the full proof and Section~\ref{sub:proof_sketch_of_theorem_ref_thm_optimal_lrs} for a proof sketch. At the boundary $s=1-1/q$, the unconstrained optimal peak LR remains of constant order, so whether the stability constraint is active depends on problem-dependent constants. The scaling analysis therefore does not distinguish between immediate and delayed decay at the boundary.

Theorem~\ref{thm: optimal_lrs} reveals a sharp phase transition at $s_c=1-1/q$. At the core is the tradeoff between signal learning and noise forgetting, together with the stability constraint: the forgetting kernel determines the preferred decay profile, while signal learning determines how much intrinsic time must be accumulated. As the task becomes harder, the unconstrained optimal peak LR increases until the stability constraint becomes active, producing the transition to delayed decay. This leads to the following interpretation:
\begin{itemize}
\item \textbf{Model capacity determines how to decay.} The capacity exponent $q$ governs the forgetting kernel $\cK$, and hence how quickly previously injected noise is forgotten. This determines the preferred decay profile. In particular, the resulting training-step decay has exponent $2q-1$, independent of the target function.
\item \textbf{Task difficulty determines when to decay.} The source exponent $s$ governs how much intrinsic time is needed for signal learning. Smaller $s$ corresponds to slower signal learning and therefore favors maintaining a large LR for longer. Once the stability constraint becomes active, this produces the extended stable phase in the hard-task regime.
\end{itemize}
The derivation in Section~\ref{sub:proof_sketch_of_theorem_ref_thm_optimal_lrs} makes this structure transparent.

\section{Fractional LR schedules: shape and rate optimality}
\label{sec:suff-cond}

In practice, one typically fixes a decay shape and tunes only the peak LR. This raises a more practical question: how much performance is lost by fixing the shape, and which decay shapes still attain the optimal convergence rate after peak-LR tuning? To answer this, we study fractional schedules whose shape is defined over relative training progress.

\begin{definition}[Fractional LR schedule]
\label{def:fractional}
A family of LR schedules $\eta_N:[0,N]\to\RR_{\ge0}$ is called \emph{fractional} if
\[
\eta_N(z)={\bar{\eta}}\,\zeta\left(\frac{z}{N}\right),\qquad z\in[0,N],
\]
where ${\bar{\eta}}>0$ may depend on $N$, while $\zeta:[0,1]\to\RR_{\ge0}$ is a profile independent of $N$, satisfying 
$
\int_0^1\zeta(u)\dd u=1.
$
\end{definition}
For a fixed training horizon $N$, the total intrinsic time is $T=\int_0^N\eta_N(z)\dd z=\bar{\eta}N$. Hence,  the schedule can be reparameterized by the total intrinsic-time budget $T$ and the allocation profile $\zeta$.

Define the cumulative profile map $P(x)\coloneqq\int_0^x\zeta(u)\dd u$. The intrinsic time at step $z$ is $t=\int_0^z\eta_N(z)\dd z=TP(z/N)$, so $t/T=P(z/N)$. Let $\psi\coloneqq\zeta\circ P^{-1}$, where $P^{-1}$ denotes a generalized inverse of $P$. Then $z/N=P^{-1}(t/T)$, and the LR expressed in intrinsic time is
\begin{equation}\label{eqn: intrinsic-time LR}
\varphi(t)=\frac{T}{N}\psi\left(\frac{t}{T}\right).
\end{equation}

Substituting this representation into the FSL~\eqref{eqn: plk_fsl}  gives the final-step loss:
\begin{equation}
\label{eqn:fractional-fixed-budget}
\cF_N(T,\zeta)\coloneqq\cS(T)+\frac{T}{N}\int_0^T\cK\!\left(T-t\right)\psi\left(\frac{t}{T}\right)\dd t.
\end{equation}
The stability constraint~\eqref{equ: functional_physical}: $\bar{\eta}\|\zeta\|_\infty\leq \eta_{\stability}$ becomes
\begin{equation}
\label{eqn:fractional-stability}
T\le\frac{N\eta_{\stability}}{\|\zeta\|_\infty}.
\end{equation}
Schedule design therefore amounts to minimizing $\cF_N(T,\zeta)$ over $T>0$ and unit-area profiles subject to~\eqref{eqn:fractional-stability}. In this section, we keep $\zeta$ fixed and optimize $T$, corresponding to the practical setting of tuning the peak LR while keeping the schedule shape fixed.

\subsection{Terminal decay and convergence rates}

We characterize the terminal behavior of $\zeta$ by assuming that, for some $p\ge0$ and $\delta\in(0,1)$,
\begin{equation}
\label{eqn:power-tail}
\zeta(x)\eqsim(1-x)^p,\qquad x\in[\delta,1).
\end{equation}
Larger $p$ means faster decay near the end of training. Several commonly used schedules satisfy this condition:
\begin{itemize}
\item \textbf{Constant schedule.} $\zeta(x)\equiv1$, corresponding to $p=0$.
\item \textbf{Cosine decay.} {$\zeta(x)=1+\cos(\pi x)\eqsim(1-x)^2$ as $x\to1$, corresponding to $p=2$.}
\item \textbf{1-sqrt decay.} {$\zeta(x)=3(1-\sqrt{x})\eqsim1-x$ as $x\to1$, corresponding to $p=1$.}
\item \textbf{Power decay.} {$\zeta(x)=(p+1)(1-x)^p$.} The case $p=1$ is linear decay~\citep{bergsma2025straight}.
\end{itemize}

\begin{theorem}[Scaling law for fractional LR schedules]
\label{thm:fractional-fsl}
Let $\zeta$ be a fixed bounded fractional profile satisfying~\eqref{eqn:power-tail}, and let ${ T\gtrsim1}$ be the total intrinsic time. Let $\alpha=\min\{q,p+1\}$. Then
\begin{equation}
\label{eqn: fractional LRS-scaling}
{\cF_N(T,\zeta)\eqsim T^{-s}+\frac{T^{1/\alpha}}{N}\bigl[{\log T}\bigr]^{\mathbf{1}\{q=p+1\}}.}
\end{equation}
\end{theorem}
The proof is given in Appendix~\ref{app:proof-thm-fractional-fsl}. Since $T=N\bar{\eta}$, increasing the intrinsic-time budget also increases the overall LR scale and hence the amount of stochastic noise injected during training. Theorem~\ref{thm:fractional-fsl} shows that the resulting noise level is governed by a competition between two effects: the forgetting kernel $\cK(t)\eqsim t^{-(2-1/q)}$ controls how much previously injected noise remains, while the terminal decay $\zeta(x)\eqsim(1-x)^p$ controls how strongly fresh noise is suppressed near the end of training. These two mechanisms contribute the exponents $1/q$ and $1/(p+1)$, respectively, so that
\[
\frac1\alpha=\max\left\{\frac1q,\frac1{p+1}\right\},\qquad \alpha=\min\{q,p+1\}.
\]
When $p+1>q$, terminal noise is sufficiently suppressed and forgetting becomes the bottleneck; when $p+1<q$, fresh noise near the end dominates and terminal decay becomes the bottleneck.

We next optimize the total intrinsic time $T$ for each profile $\zeta$. Let $T^*$ be a minimizer of $\cF_N(T,\zeta)$ subject to~\eqref{eqn:fractional-stability}, and let $\cE_N^*$ denote the resulting final-step FSL loss.

\begin{highlightedtheorem}
\begin{theorem}[{Optimal intrinsic time for fractional schedules}]
\label{thm: optimal-fractional-lrs}
Fix a {bounded} fractional profile $\zeta$ with tail exponent $p$, and let $\alpha=\min\{q,p+1\}$. Then:
\begin{itemize}
\item If $s\ge1-\frac1\alpha$, then for $q\neq p+1$,
\[
{T^*\eqsim N^{\frac{1}{s+1/\alpha}}},
\qquad
\cE_N^*\eqsim N^{-\frac{s\alpha}{s\alpha+1}}.
\]
For $q=p+1$, these scalings acquire logarithmic corrections:
\[
{T^*\eqsim \left(\frac{N}{\log N}\right)^{\frac{1}{s+1/\alpha}}},
\qquad
\cE_N^*\eqsim N^{-\frac{s\alpha}{s\alpha+1}}(\log N)^{\frac{s\alpha}{s\alpha+1}}.
\]
\item If $s<1-\frac1\alpha$,
\[
{T^*\eqsim N},
\qquad
\cE_N^*\eqsim N^{-s}.
\]
\end{itemize}
\end{theorem}
\end{highlightedtheorem}
The proof is given in Appendix~\ref{app:proof-thm-opt-fractional}. Theorem~\ref{thm: optimal-fractional-lrs} has the same two-regime structure as Theorem~\ref{thm: optimal_lrs}, with $q$ replaced by the effective exponent $\alpha=\min\{q,p+1\}$. Thus, fixing the schedule shape can shift the regime boundary from $s=1-1/q$ to $s=1-1/\alpha$. Two consequences follow:
\begin{itemize}
\item \textbf{Optimal rates do not require the exact optimal shape.} After optimizing $T$, any fractional profile with $p+1>q$ attains the optimal rate in both regimes. Thus, the terminal exponent need not match the exact optimum $2q-1$; in the hard-task regime, a long stable phase is also unnecessary for attaining $N^{-s}$.
\item \textbf{Slow terminal decay causes capacity saturation.} When $p+1<q$, we have $\alpha=p+1$, so the optimized rate no longer improves with $q$. We call this \emph{schedule-induced capacity saturation}. The schedule remains rate-optimal when $s\le1-1/(p+1)$, but becomes suboptimal when $s>1-1/(p+1)$. Figure~\ref{fig: intro}(right) summarizes these regimes.
\end{itemize}

\paragraph{Implications for commonly used schedules.}
The criterion $p+1>q$ directly gives the regimes in which standard schedules attain the optimal rate after optimizing $T$. Cosine decay has $p=2$ and is therefore rate-optimal for all tasks when $q<3$. Linear and 1-sqrt decay have $p=1$ and are similarly rate-optimal for all tasks when $q<2$. For power decay, $p$ is a design parameter, and any $p>q-1$ attains the optimal rate without matching the exact optimal exponent $2q-1$ given in Theorem~\ref{thm: optimal_lrs}.

\paragraph{Numerical validation.}
Figure~\ref{fig: intro}(middle) validates the predicted capacity saturation for PLK regression with $(s,q)=(0.8,5)$: power decay with $p=4.2$ follows the optimal rate, whereas cosine decay ($p=2$) exhibits the predicted saturated rate. Additional schedule families and regimes are tested in Appendix~\ref{app:additional-experiments}.

\subsection{Why relative training progress matters}

Fractional schedules define the shape over the relative progress $z/N$. We now show that fixing a profile in absolute training steps can prevent rate optimality, even when the peak LR is tuned for each horizon. Classical polynomial decay provides a simple illustration:
\[
\eta_N(z)=\eta_0(1+z)^{-a},
\]
{where $0<\eta_0\lesssim1$ and $a>0$ is fixed as $N$ varies.} For $a\ge1$, the total intrinsic time grows at most logarithmically, so the signal-learning term alone is already suboptimal.  It therefore suffices to consider $a\in(0,1)$, for which
\[
T_N=\int_0^N\eta_N(z)\dd z\eqsim\eta_0N^{1-a}.
\]
Changing the decay exponent now changes the power of $N$ in the total intrinsic time, whereas the fractional family has the same total intrinsic time for every profile at a fixed LR scale after the unit-integral normalization.

The terminal LR satisfies
$
\eta_N(N)\eqsim\eta_0N^{-a}\eqsim T_N/N.
$
Since the schedule is non-increasing, for $T_N\ge1$ the noise term is bounded below by the terminal LR:
\[
\int_0^{T_N}\cK(T_N-\tau)\varphi(\tau)\dd\tau
\geq \eta_N(N)\int_0^{T_N}\cK(T_N-\tau)\dd\tau
\eqsim\eta_N(N).
\]
Consequently,
\[
\cF[\eta_N]\gtrsim T_N^{-s}+\frac{T_N}{N}
\ge\min_{T>0}\left(T^{-s}+\frac{T}{N}\right)
\eqsim N^{-\frac{s}{s+1}}.
\]
This is strictly slower than the optimal rate in both the easy- and hard-task regimes. Thus, polynomial decay cannot simultaneously accumulate enough intrinsic time for signal learning and sufficiently suppress late-stage noise injection.

This limitation extends beyond polynomial decay. Consider families $\eta_N(z)=\eta_0{h}(z)$, where the profile ${h}$ is independent of $N$, while the scale $\eta_0$ may be tuned separately for each training horizon. We call such profiles \emph{horizon-independent}. The following theorem establishes a rate barrier for every fixed non-increasing profile satisfying mild regularity conditions, whose proof is deferred to Appendix~\ref{app:proof-horizon-free-lower-bound}.

\begin{theorem}[Rate barrier for horizon-independent profiles]
\label{thm:horizon-free-lower-bound}
Let $\eta_N(z)=\eta_0{h}(z)$, where $0<\eta_0\lesssim1$ may depend on $N$, and ${h}:[0,\infty)\to(0,1]$ is locally integrable, non-increasing, independent of $N$, and satisfies ${h}(0)=1$. Then, for every $\rho>s/(s+1)$,
\[
\limsup_{N\to\infty}N^\rho\cF[\eta_N]=\infty.
\]
\end{theorem}
Thus, peak-LR tuning alone cannot attain the optimal rate for any fixed non-increasing profile in absolute training steps.

\section{Discrete-time SGD for kernel regression}
\label{sec:kernel-sgd}

We now turn from the continuous-time FSL analysis to rigorous convergence guarantees for discrete-time SGD. We work under the standard source and capacity conditions commonly used in kernel regression, which retain the exponents $s$ and $q$ from Assumption~\ref{ass:power-law} but replace the exact power-law structure by more general spectral regularity conditions. This also allows direct comparison with the existing kernel-regression literature.
\begin{assumption}[Source condition]\label{ass: source}
For some $s>0$, 
$
f^*=\sum_{j\ge1}a_j\lambda_j^{(s-1)/2}\phi_j
$
with $\sum_{j\ge1}a_j^2\le1$.
\end{assumption}

\begin{assumption}[Capacity condition]\label{ass: capacity}
There exists a $q>1$ such that $\lambda_j\lesssim j^{-q}$ for all $j\in\NN_{+}$.
\end{assumption}

Under these conditions, classical results establish the minimax lower bound~\citep{caponnetto2007optimal,steinwart2009optimal} 
\begin{equation}
\label{eqn: minimax-optimal}
\inf_{\widehat{\btheta}_N}\sup_{\cD}\EE_{\mathcal{D}_N}\!\left[\cE(\widehat{\btheta}_N)\right]\gtrsim N^{-\frac{sq}{sq+1}},
\end{equation}
where $\mathcal{D}_N=\{(\bx_k,y_k)\}_{k=0}^{N-1}$ consists of $N$ i.i.d.\ samples, the infimum is over all estimators measurable with respect to $\mathcal{D}_N$, and the supremum is over distributions satisfying Assumptions~\ref{ass: source} and~\ref{ass: capacity}. We ask how closely the last iterate of one-pass SGD can approach this statistical limit.

\begin{theorem}[Convergence rate of SGD with power decay] \label{thm:power-rate}
Suppose Assumptions~\ref{ass: hypercontra},  \ref{ass: source}, and \ref{ass: capacity} hold. Consider LR schedule $\eta_k = \eta_0 (1-k/N)^{p}$.  The following statements hold.
\begin{itemize}
    \item \textbf{Easy-task regime} ($s \ge 1 - \frac{1}{q}$).
    Choosing $\eta_{0} \eqsim N^{-\frac{s-1+1/q}{s+1/q}}$ and
    $p > q-1$ yields
    \[
        \EE\!\left[\cE(\btheta_N)\right]
        \;\lesssim\;
        N^{-\frac{sq}{sq+1}} .
    \]

    \item \textbf{Hard-task regime} ($s < 1 - \frac{1}{q}$).
    Choosing $\eta_{0} \eqsim 1$ and $p > \frac{s}{1-s}$ yields
    $$
        \EE\!\left[\cE(\btheta_N)\right]
        \;\lesssim\;
        N^{-s} .
    $$
\end{itemize}
\end{theorem}

Here \(\eta_0\) denotes the initial LR. The proof is deferred to Appendix~\ref{app:proof-theorem-power-cosine}. This theorem shows that with the power-decay schedule, SGD can reach 
 the minimax-optimal rate~\eqref{eqn: minimax-optimal} in the easy-task regime.  In the hard-task regime,  the obtained rate is also optimal, in the sense that it coincides with the best rate achievable by one-pass SGD:

\begin{proposition} \label{prop:lower-bound-any-lrs}
Suppose that $\max_{i\in[N]} \eta_i \lesssim 1$.
If $s < 1 - \frac{1}{q}$,  then
$
    \sup_{\cD}\, \EE\!\left[\cE(\btheta_N)\right]
    \;\gtrsim\;
    N^{-s}.
$
\end{proposition}
Here   $\sup_{\cD}$ denotes the supremum over all data distributions satisfying Assumptions~\ref{ass: hypercontra}, \ref{ass: source}, and \ref{ass: capacity}.
 The proof of this proposition is deferred to Appendix~\ref{app:proof-prop-lower-bound-any-lrs}.

To the best of our knowledge, Theorem~\ref{thm:power-rate} is the first last-iterate SGD result to match the minimax-optimal rate in the easy-task regime without logarithmic factors, closing the gap with averaged SGD. This improvement is motivated by the FSL schedule-design analysis, which identifies power decay as the relevant family, rather than prescribing an exponential-decay schedule a priori and then analyzing it~\citep{lin2024scaling,zhang2024optimality}.

\section{Experiments with nonlinear neural networks}
\label{sec:deep-experiments}

Our theory suggests that the preferred decay duration depends on task difficulty: easy tasks favor longer decay phases, whereas hard tasks favor short ones. We test whether this pattern persists in  neural network training using a stable--linear-decay schedule.

\subsection{Nonlinear regression with two-layer MLPs}
\label{subsec:mlp-fourier}

We first consider nonlinear regression on the unit circle. Each input is parameterized as $\boldsymbol{x}=(\cos\theta,\sin\theta)^\top\in\mathbb{S}^1$, and the target function has power-law Fourier coefficients,
\begin{equation}
\label{eqn:fourier-target}
f_\gamma^*(\boldsymbol{x})=\sum_{k=1}^{\infty}k^{-\gamma}\cos(k\theta).
\end{equation}
In the experiments, the series is truncated at a sufficiently large frequency. The exponent $\gamma$ controls target smoothness: smaller $\gamma$ places more energy in high-frequency components and produces a harder learning problem, whereas larger $\gamma$ gives a smoother and easier target.

We train a two-layer ReLU MLP for $1000$ steps using one-pass SGD with batch size $1$. We  fix the peak LR at $0.1$ and vary only the decay ratios. Concretely, we scan decay ratios $r\in\{0.1,0.2,\ldots,1.0\}$ and report the mean over $32$ runs for each ratio. Figure~\ref{fig:mlp-fourier} shows the final validation loss. The harder target $\gamma=0.6$ prefers an intermediate decay ratio $r=0.5$, whereas the easier target $\gamma=10$ prefers decay from the start ($r=1$), consistent with our prediction.  Additional results with a different peak LR are provided in Appendix~\ref{app:additional-mlp-lr005}.

\begin{figure}[!h]
\centering
    \begin{minipage}[b]{0.43\textwidth}
        \centering
        \includegraphics[width=\linewidth]{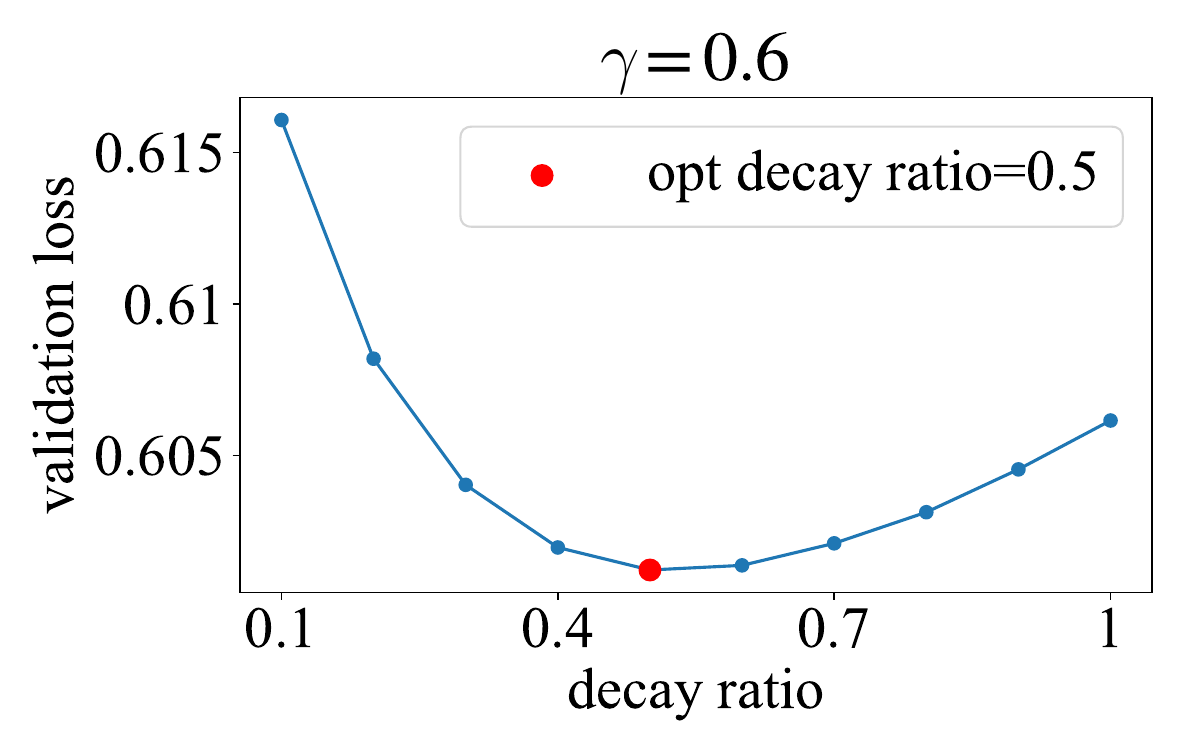}
    \end{minipage}
    \hspace*{2em}
    \begin{minipage}[b]{0.43\textwidth}
        \centering
        \includegraphics[width=\linewidth]{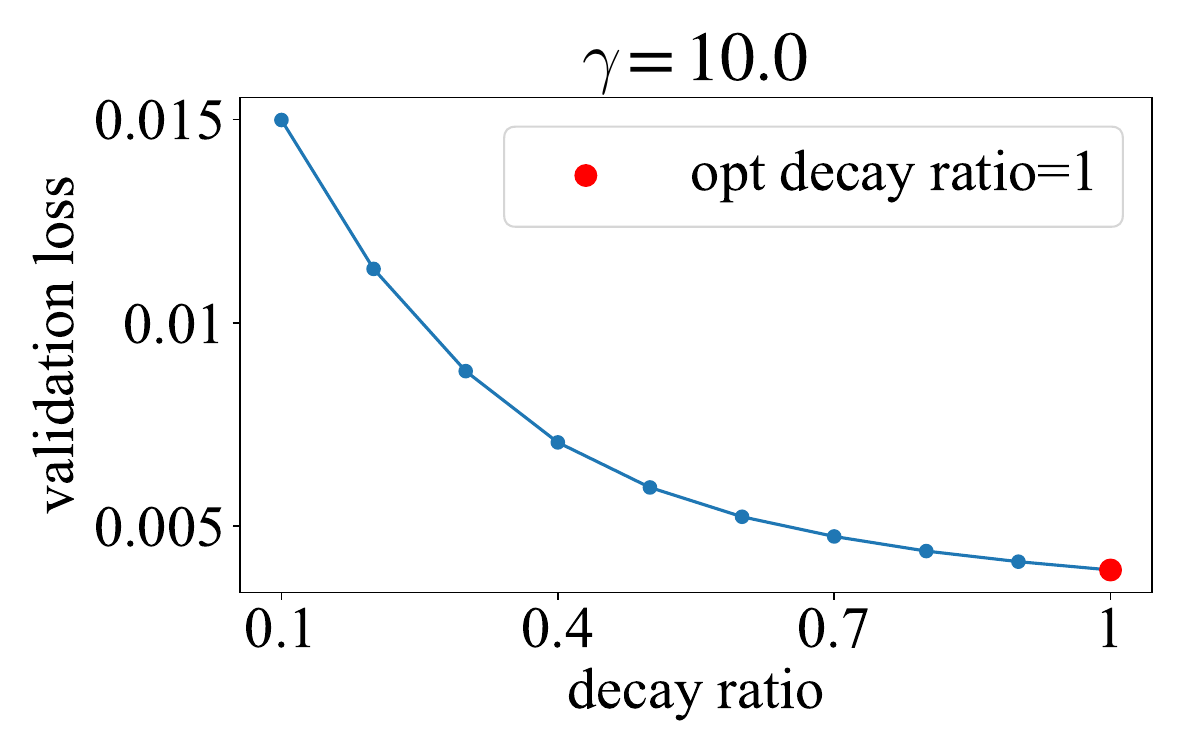}
    \end{minipage}
    \vspace*{-.7em}
\caption{\small
\textbf{Decay ratio versus validation loss for nonlinear MLPs.}
The harder target $\gamma=0.6$ prefers $r=0.5$, while the easier target $\gamma=10$ prefers $r=1$, corresponding to decay from the start of training. Points show means over $32$ runs; red markers indicate the grid minimizers. The preferred decay duration increases as the task becomes easier, consistent with our theory.
}
\label{fig:mlp-fourier}
\end{figure}

\subsection{Full versus binary classification on CIFAR-100}
\label{subsec:cifar100}

We next test our predictions on image classification using two tasks constructed from the same CIFAR-100 images: CIFAR-Full and CIFAR-Binary. CIFAR-Full uses the original 100 fine-grained classes, whereas CIFAR-Binary groups them into animals and non-animals, making CIFAR-Full the harder task in this comparison.

For both tasks, we train ResNet-18 from scratch using the same train/validation split and optimization setup: SGD with momentum $0.9$, weight decay $10^{-4}$, batch size $20$, peak LR $0.2$, and $5$ training epochs. We use linear warmup for $10\%$ of training, followed by a stable--linear-decay schedule, and vary the post-warmup decay ratio $r\in\{0.2,0.4,\ldots,1.0\}$. Figure~\ref{fig:cifar100} shows that CIFAR-Full prefers a shorter decay phase ($r=0.4$), whereas CIFAR-Binary prefers decay immediately after warmup ($r=1$), consistent with our prediction that harder tasks favor longer stable phases. Additional results with different peak LRs and batch sizes are provided in Appendix~\ref{app:additional-cifar100}.

\begin{figure}[t]
\centering
    \begin{minipage}[b]{0.43\textwidth}
        \centering
        \includegraphics[width=\linewidth]{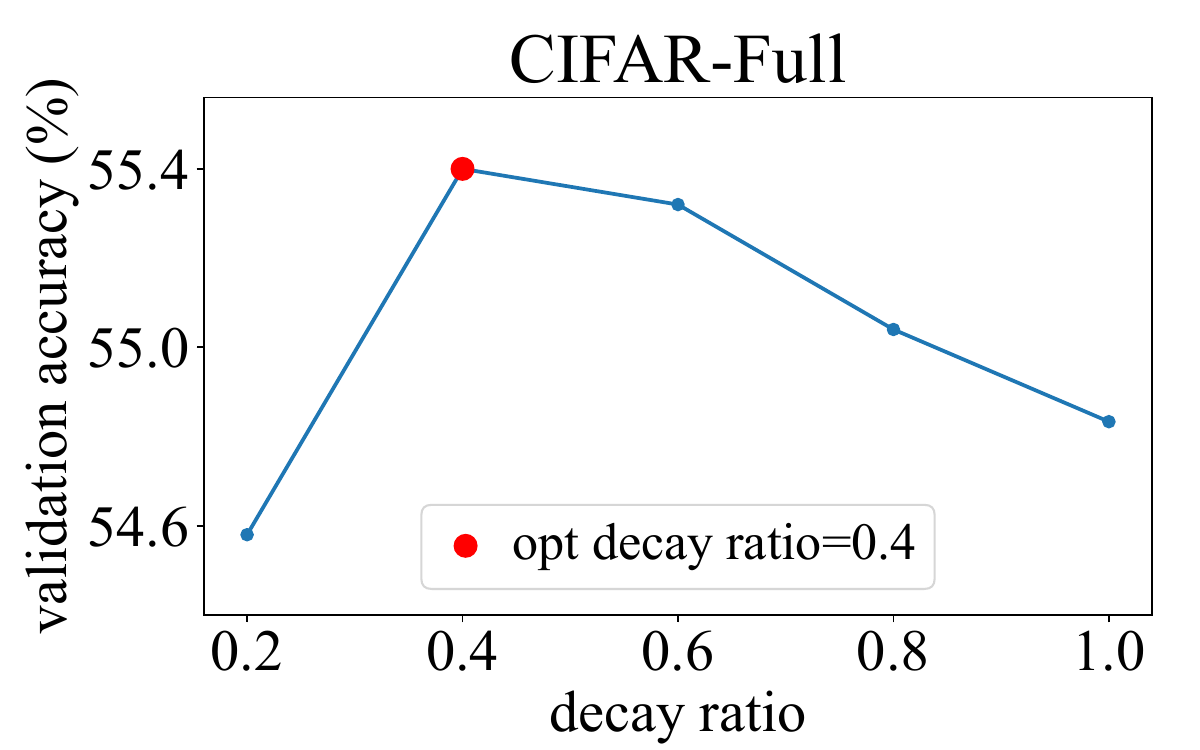}
    \end{minipage}
    \hspace*{2em}
    \begin{minipage}[b]{0.43\textwidth}
        \centering
        \includegraphics[width=\linewidth]{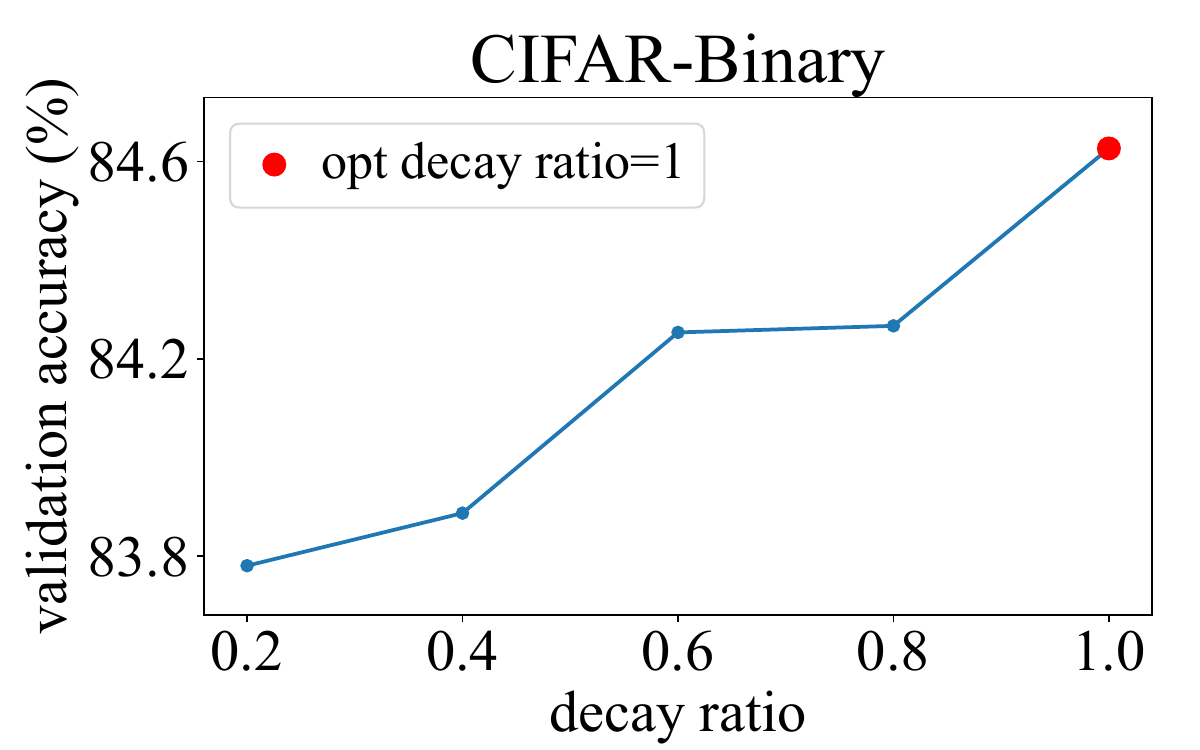}
    \end{minipage}
     \vspace*{-.7em}
\caption{\small
\textbf{Decay ratio versus final validation accuracy on CIFAR-100.}
CIFAR-Full is maximized at $r=0.4$, while CIFAR-Binary is maximized at $r=1$. Points show means over three runs; red markers indicate the best decay ratios on the tested grid.
}
\label{fig:cifar100}
\end{figure}

\section{Proof sketches and interpretation}

\subsection{Proof sketch of Theorem~\ref{thm: optimal_lrs}}
\label{sub:proof_sketch_of_theorem_ref_thm_optimal_lrs}

We solve the variational problem~\eqref{equ: functional_physical} in three steps.

\paragraph{Step 1: Optimize the profile for a fixed intrinsic-time budget.}
Fix $t(N)=T$. Recall that
$
t(z)=\int_0^z\eta(u)\dd u
$
and $\varphi(\tau)=\eta(t^{-1}(\tau))$. Under the change of variables $\tau=t(z)$, we have
$
\dd\tau=t'(z)\dd z=\eta(z)\dd z=\varphi(\tau)\dd z,
$
and hence
$
\dd z=\dd\tau/\varphi(\tau).
$
Therefore, the fixed training horizon $N$ becomes the constraint
\[
N=\int_0^N\dd z=\int_0^T\frac{1}{\varphi(\tau)}\dd\tau.
\]
Since the signal term $\cS(T)$ is fixed, we only need to minimize the noise term:
\begin{equation}
\label{equ: noise_term_intrinsic-0}
\min_{\varphi:[0,T]\to\RR_{>0}}\quad \int_0^T\cK(T-\tau)\varphi(\tau)\dd\tau
\qquad\text{s.t.}\qquad
\int_0^T\frac{1}{\varphi(\tau)}\dd\tau=N.
\end{equation}
By Cauchy--Schwarz,
\[
\left(\int_0^T\cK(T-\tau)\varphi(\tau)\dd\tau\right)\left(\int_0^T\frac{1}{\varphi(\tau)}\dd\tau\right)\ge\left(\int_0^T\sqrt{\cK(T-\tau)}\dd\tau\right)^2.
\]
Equality yields the optimal intrinsic-time profile
\begin{equation}
\label{equ: optimal_lrs_intrinsic-0}
\varphi^*(\tau)=\frac{\int_0^T\sqrt{\cK(T-u)}\dd u}{N\sqrt{\cK(T-\tau)}}\propto \cK(T-\tau)^{-1/2}=(c_K+T-\tau)^{1-\frac{1}{2q}}.
\end{equation}
Thus, for a fixed intrinsic-time budget, the \emph{optimal LR profile is determined entirely by the forgetting kernel}. Converting this profile back to training-step time yields the asymptotic power-decay shape
\[
\eta^*(z)\propto\left(1+\delta_N-\frac{z}{N}\right)^{2q-1},\qquad \delta_N=o_N(1).
\]
\paragraph{Step 2: Optimize the total intrinsic time.}
Substituting~\eqref{equ: optimal_lrs_intrinsic-0} into~\eqref{equ: noise_term_intrinsic-0} gives
\[
\int_0^T\cK(T-\tau)\varphi^*(\tau)\dd\tau
=\frac{1}{N}\left(\int_0^T\sqrt{\cK(T-\tau)}\dd\tau\right)^2
\eqsim\frac{T^{1/q}}{N}.
\]
Hence, after optimizing the profile for fixed $T$, the final-step loss reduces to
\begin{equation}
\label{eqn: oo2}
\cE_{N,T}\eqsim T^{-s}+\frac{T^{1/q}}{N}.
\end{equation}
Balancing the two terms yields
\[
T^*\eqsim N^{\frac{q}{1+sq}}.
\]

\paragraph{Step 3: Impose the stability constraint.}
It remains to check whether the LR schedule corresponding to $T^*$ satisfies the stability constraint. From~\eqref{equ: optimal_lrs_intrinsic-0},
\[
\varphi^*(0)
=
\frac{\int_0^T\sqrt{\cK(T-u)}\dd u}{N\sqrt{\cK(T)}}
\eqsim
\frac{T^{1/(2q)}}{N}\,T^{1-\frac{1}{2q}}
=
\frac{T}{N}.
\]
Substituting $T=T^*$ therefore gives the corresponding peak LR:
\[
\varphi^*(0)\eqsim\frac{T^*}{N}
\eqsim
N^{\frac{q}{1+sq}-1}
=
N^{-\frac{s+1/q-1}{s+1/q}}.
\]
This scaling changes qualitatively at $s=1-1/q$. If $s>1-1/q$ (the easy-task regime), then $\varphi^*(0)\to0$, so the stability constraint is asymptotically inactive and the power-decay profile from Step~1 is feasible from the beginning of training. 

\begin{wrapfigure}{r}{0.35\textwidth}
\centering
\includegraphics[width=\linewidth]{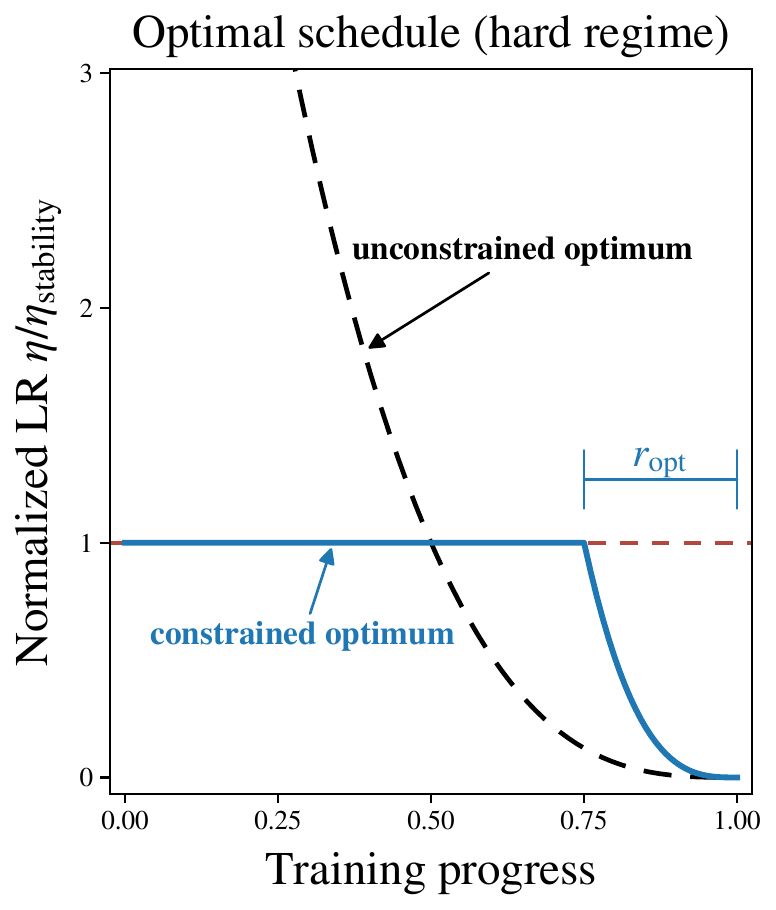}
\vspace{-2em}
\end{wrapfigure}
If $s<1-1/q$ (the hard-task regime), the unconstrained peak LR $\varphi^*(0)$ diverges with $N$ and eventually exceeds the stability threshold $\eta_{\stability}$. The stability constraint therefore becomes active. The KKT conditions for~\eqref{equ: functional_physical}, together with monotonicity of the optimal profile, imply that the constrained optimum stays at $\eta_{\stability}$ over an initial interval and then follows the same power-law decay profile as in~\eqref{equ: optimal_lrs_intrinsic-0}. Thus, the optimal schedule has a stable-decay structure. As illustrated in the right figure, the unconstrained profile exceeds the stability ceiling, while the constrained optimum replaces this infeasible early portion by a stability-saturated plateau followed by decay.

Let $r\in(0,1)$ denote the fraction of training devoted to the decay phase. The stable and decay phases then span $N(1-r)$ and $Nr$ iterations, respectively, and the schedule takes the asymptotic form
\[
\eta_r(z)
=
\eta_{\stability}
\begin{cases}
1, & 0\leq z\leq N(1-r),\\[3pt]
\displaystyle {\chi}\left(\frac{z-N(1-r)}{Nr}\right), & N(1-r)<z\leq N,
\end{cases}
\qquad
{\chi}(u)=(1-u)^{2q-1}.
\]

Increasing $r$ lengthens the decay phase and improves noise forgetting,
but reduces the total intrinsic time available for signal learning. Let
\[
c_{\chi}
:=
\int_0^1 {\chi}(u)\dd u
=
\frac{1}{2q},
\qquad
c':=1-c_{\chi},
\qquad
\eta_0:=\eta_{\stability}.
\]
The stable and decay phases contribute intrinsic times
$\eta_0N(1-r)$ and $\eta_0c_{\chi} Nr$, respectively. Hence,
\[
T_N
=
\eta_0N(1-r)+\eta_0c_{\chi} Nr
=
\eta_0N(1-c'r).
\]

Let $T_{N,1}=\eta_0N(1-r)$ denote the intrinsic time at the onset of
decay. Splitting the noise term at $T_{N,1}$ gives
\[
\begin{aligned}
\int_0^{T_N}
\cK(T_N-\tau)\varphi(\tau)\dd\tau
={}&
\eta_0\int_0^{T_{N,1}}
\cK(T_N-\tau)\dd\tau +
\int_{T_{N,1}}^{T_N}
\cK(T_N-\tau)\varphi(\tau)\dd\tau.
\end{aligned}
\]
Recall that $s_c=1-\frac{1}{q}$.
Since $T_N-T_{N,1}=\eta_0c_{\chi} Nr$ and
$\cK(v)\eqsim v^{-({s_c}+1)}$, the change of variables
$v=T_N-\tau$ yields
\[
\eta_0\int_0^{T_{N,1}}
\cK(T_N-\tau)\dd\tau
=
\eta_0\int_{\eta_0c_{\chi} Nr}^{T_N}
\cK(v)\dd v
\lesssim
\frac{\eta_0}{(\eta_0c_{\chi} Nr)^{s_c}}.
\]
Meanwhile, applying the power-decay result from Step~1 to the
final $Nr$ iterations gives
\[
\int_{T_{N,1}}^{T_N}
\cK(T_N-\tau)\varphi(\tau)\dd\tau
\eqsim
\frac{\eta_0}{(\eta_0Nr)^{s_c}}
\eqsim
\frac{\eta_0}{(\eta_0c_{\chi} Nr)^{s_c}},
\]
where the last equivalence uses the fact that $c_{\chi}$ is independent
of $N$ and $r$. Thus, the residual noise generated during the stable
phase is absorbed into the decay-phase noise term. Combining the two
noise contributions, the terminal excess risk satisfies
\begin{equation}
\label{equ: hard_regime_risk}
\cE_N(r)
\eqsim
\frac{1}{\left(\eta_0N(1-c'r)\right)^s}
+
\frac{\eta_0}{\left(\eta_0c_{\chi} Nr\right)^{s_c}}.
\end{equation}

Since $s<{s_c}$, for every constant $r\in(0,1)$, the signal-learning term
scales as $N^{-s}$ and dominates the noise-forgetting term, which scales
as $N^{-{s_c}}$. The optimal decay fraction must therefore satisfy
$r_{\mathrm{opt}}=o_N(1)$. Suppressing constants independent of $N$ and
$r$, balancing the marginal costs in
\eqref{equ: hard_regime_risk} gives
\[
N^{-s}(1-c'r)^{-s-1}
\eqsim
N^{-{s_c}}r^{-{s_c}-1}.
\]
Consequently,
\[
r_{\mathrm{opt}}
\eqsim
N^{-\frac{{s_c}-s}{{s_c}+1}}
=
N^{-\frac{1-1/q-s}{2-1/q}}.
\]
Although the optimal decay fraction vanishes,
\[
Nr_{\mathrm{opt}}
\eqsim
N^{\frac{1+s}{2-1/q}}
\longrightarrow\infty,
\]
so the decay phase still contains a diverging number of iterations.
The resulting excess risk satisfies
\[
\cE_N(r_{\mathrm{opt}})
\eqsim
N^{-s}.
\]

\subsection{Proof sketch of Theorem~\ref{thm:fractional-fsl}}
\label{sub:proof-sketch-fractional}

We first transform the training-step LR schedule into its intrinsic-time counterpart, denoted as $\varphi(t)=(T/N)\psi(t/T)$. Crucially, the fractional structure is preserved in intrinsic time. Specifically, if the training-step profile exhibits a power-decay tail, the intrinsic-time profile also possesses a power-decay tail, satisfying $\psi(y)\eqsim(1-y)^{p/(p+1)}$ for $y\in[\widetilde\delta,1)$ and some $\widetilde\delta\in(0,1)$. Applying the FSL~\eqref{eqn:fractional-fixed-budget}, for ${ T\gtrsim1}$ we have
\[
\cF_N(T,\zeta)\eqsim T^{-s}
+\frac{T^{1/q}}{N}\int_0^1
\frac{\psi(y)}{(T^{-1}+1-y)^{2-1/q}}\dd y.
\]
We decompose the integral into two regions: $[0,\widetilde\delta]$ and $[\widetilde\delta,1]$. The asymptotic behavior is \textbf{dominated by the tail integral} over $[\widetilde\delta,1]$. We analyze the convergence based on the exponent comparison:
\begin{itemize}
\item \textbf{Fast decay regime} ($p/(p+1)>1-1/q$). In this case, the combined exponent satisfies $p/(p+1)-(2-1/q)>-1$, ensuring the limiting integral converges absolutely:
\[
\int_{\widetilde\delta}^1\frac{\psi(y)}{(T^{-1}+1-y)^{2-1/q}}\dd y\eqsim1,
\]
giving noise of order $T^{1/q}/N$.
\item \textbf{Slow decay regime} ($p/(p+1)<1-1/q$). Here, the limiting integral diverges as $y\to1$. The term $T^{-1}$ acts as a cutoff, yielding:
\[
\int_{\widetilde\delta}^1\frac{\psi(y)}{(T^{-1}+1-y)^{2-1/q}}\dd y
\eqsim T^{1-1/q-p/(p+1)}.
\]
Multiplying by $T^{1/q}/N$ gives noise of order $T^{1/(p+1)}/N$.
\end{itemize}
At the boundary $p/(p+1)=1-1/q$, equivalently $p+1=q$, the tail integral is of order $\log T$. Combining these estimates with the signal term yields~\eqref{eqn: fractional LRS-scaling}.

Intuitively, when the LR tail decays sufficiently fast, late-stage noise injection is strongly suppressed. The noise term is therefore dominated by noise injected earlier and not fully forgotten by the end of training, so its scaling is controlled by the forgetting kernel. When the LR tail decays more slowly, fresh noise injected near the end dominates instead.

\section{Conclusion and discussion}

Motivated by the empirical success of WSD, we study optimal LR scheduling within the FSL framework. We show that the optimal schedule is problem-dependent: hard tasks favor late decay, whereas easy tasks favor early decay. This provides an explanation for why WSD works well in some regimes but not others~\citep{semenov2025benchmarkingoptimizerslargelanguage,wen2026fantastic}. At the core is the tradeoff between signal learning and noise forgetting, together with the stability constraint: as the task becomes harder, the unconstrained optimal peak LR increases until the stability constraint becomes active, producing the transition to WSD. Beyond determining when decay should begin, the FSL analysis shows that the optimal decay shape is determined entirely by how the model forgets previously injected noise, independent of the target function.

In addition, we show that the exact optimal shape is not necessary for achieving the optimal scaling rate. After tuning the peak LR, a broad class of schedules with sufficiently fast terminal decay attains the optimal rate, whereas slower decay can become the bottleneck and lead to schedule-induced capacity saturation. This distinction between exact shape optimality and rate optimality helps explain why different practical decay profiles can perform similarly in some regimes but differ substantially in others.

It is worth highlighting that our FSL-based analysis of LR scheduling is macroscopic: once the signal-learning function \(S\) and forgetting kernel \(K\) are specified, the analysis no longer depends on the microscopic structure of kernel regression. Although LLM pretraining and kernel regression have very different microscopic dynamics, both have been shown to follow FSLs~\citep{li2025functional}. This suggests that the schedule-design principles developed here may extend to practical LLM pretraining. Several challenges, however, remain:
\begin{itemize}
\item \textbf{What determines the training regime?} In our analysis, task difficulty determines whether early or late decay is preferred. The relevant regime variable in practical pretraining may be different. Batch size is one concrete candidate: increasing the batch size reduces noise injection, but under a fixed data budget it also reduces the number of optimization steps and hence the accumulated intrinsic time~\citep{wang2026fast}.
\item \textbf{What is the effective control variable?} Our analysis treats the LR as the control variable, whereas recent work suggests that the effective learning rate, rather than the nominal LR, more directly governs loss dynamics~\citep{liu2026effective}.
\end{itemize}
Extending the present framework to practical pretraining will require incorporating these effects.

\section*{Acknowledgement}

Lei Wu is supported in part by the National Natural Science Foundation of China (NSFC12522120, NSFC92470122, and NSFC12288101). Binghui Li is supported by the Elite Ph.D.~Program in Applied Mathematics at Peking University. We thank Shengtao Guo for fruitful discussions during the early stages of this work.

\bibliography{wsd}
\bibliographystyle{ims}

\newpage

\appendix

\part{Appendix}
\parttoc

\newpage
\section{Proof of Theorem~\ref{thm: optimal_lrs}}
\label{appendix: proof_opt_lrs}

In this section, we provide a detailed proof of Theorem~\ref{thm: optimal_lrs}.  Define the auxiliary functions:
\[
\quad L(T,t,t') = \cK(T-t)(t')^2. 
\] 
The objective functional~\eqref{equ: functional_physical} can be rewritten using the intrinsic time function $t$ as follows:
\begin{equation}
\label{eqn: fsl-intrsinc-time}
\begin{aligned}
    \min_{t \in \operatorname{AC}([0,N]), T} \quad & \widetilde{\cF}[t, T] := S(T) + \int_0^N L(T,t(z),t'(z)) \dd z
    \\
    \text{s.t. } \quad & t(0)=0, \; t(N) = T,
    \\
    & 0\le t'(z) \leq \eta_{\stability} \quad\text{for a.e.\ }z \in[0,N].
\end{aligned}
\end{equation}

Directly solving the variational problem~\eqref{eqn: fsl-intrsinc-time} is challenging because of the nonlinear dependence on the LR schedule. We therefore decompose the proof into three steps. Since we focus on the large-$N$ scaling regime, the $O(1)$ offsets $c_S$ and $c_K$ do not affect the asymptotic scaling or the structure of the optimal schedule. For notational simplicity, throughout this appendix we normalize $c_S=c_K=1$ and write
\[
S(t)=(1+t)^{-s},\qquad K(t)=(1+t)^{-(2-1/q)}.
\]

\subsection{Step 1: Deriving the optimal profile under a fixed intrinsic-time budget}

We first fix the total intrinsic time $t(N) = T$ and optimize the schedule profile subject \textit{solely} to this constraint (temporarily omitting the peak learning rate constraint). Formally, we apply the change of variables $z = t^{-1}(\tau)$. By the inverse function theorem, the differential transforms as:
\begin{equation*}
    \dd z = \frac{\dd \tau}{t'(z)} = \frac{\dd \tau}{\eta(z)} = \frac{\dd \tau}{\varphi(\tau)}. 
\end{equation*}
Then, the requirement imposes the following integral constraint:
\begin{equation*}
    N = \int_{0}^{N} 1 \dd z = \int_{0}^{T} \frac{1}{\varphi(\tau)} \dd \tau, 
\end{equation*}
which leads to the following variational problem of noise term:
\begin{equation}
\begin{aligned}
\label{equ: noise_term_intrinsic-1}
    \min_{\varphi:[0, T]\to \RR_{\geq 0}} \quad & \int_{0}^{T} \cK(T-\tau)\varphi(\tau) \dd \tau
    \\
    \text{s.t.} \quad & \int_{0}^{T} \frac{1}{\varphi(\tau)} \dd \tau = N.
\end{aligned}
\end{equation}
Applying the Cauchy-Schwarz inequality yields
\begin{equation*}
    \left(\int_{0}^{T} \cK(T - \tau) \varphi(\tau) \dd \tau \right) \left(\int_{0}^{T}\frac{1}{\varphi(\tau)}\dd \tau\right) \geq \left(\int_{0}^{T} \sqrt{\cK(T - \tau)} \dd \tau \right)^2. 
\end{equation*}
Consequently, the noise term is minimized when the equality condition holds, which yields the optimal intrinsic profile:
\begin{equation}
\label{equ: optimal_lrs_intrinsic}
\begin{aligned}
    \varphi(\tau) &= \frac{\int_{0}^{T}\sqrt{\cK(T-u)}\dd u}{N\sqrt{\cK(T-\tau)}}
    \\
    &= a_{N,T} (1 + T - \tau)^{1-\frac{1}{2q}}
\end{aligned}
\end{equation}
where $a_{N,T}:=\frac{2q}{N}((1+T)^{\frac{1}{2q}} - 1)$. We have used $\cK(t)=(1+t)^{1/q-2}$.

Now, we proceed to derive the LR schedule in terms of training steps $z$. Recall that $\frac{\dd \tau}{\dd z} = \varphi(\tau)$. Substituting the profile from \eqref{equ: optimal_lrs_intrinsic}, we have
\begin{equation}
\label{equ: lrs_ode}
    \frac{\dd \tau}{\dd z} = \varphi(\tau) = a_{N,T} (1 + T - \tau)^{1-\frac{1}{2q}},
\end{equation}
with boundary condition $\tau = 0$ when $z = 0$.

Solving the ODE~\eqref{equ: lrs_ode}, we derive
\begin{equation*}
    z = a_{N,T}^{-1} 2q (1+T)^{\frac{1}{2q}} -a_{N,T}^{-1} 2 q (1 + T -\tau)^{\frac{1}{2q}}, 
\end{equation*}
which implies the following intrinsic time function:
\begin{equation*}
    t(z) = 1 + T - \left((1+T)^{\frac{1}{2q}} - \frac{a_{N,T}}{2q} z \right)^{2q}. 
\end{equation*}
Differentiating with respect to $z$ gives the training-step LR schedule:
\begin{equation*}
    \eta(z) = t'(z) = \frac{2q}{N}\left((1+T)^{\frac{1}{2q}} - 1\right)^{2q}\left(1 + \frac{1}{(1+T)^{\frac{1}{2q}}-1} - \frac{z}{N}\right)^{2q-1} . 
\end{equation*}

\subsection{Step 2: Determining the optimal intrinsic-time budget}

Substituting the optimal intrinsic profile~\eqref{equ: optimal_lrs_intrinsic} into the noise term~\eqref{equ: noise_term_intrinsic-1}, we derive the accumulated variance:
\begin{equation*}
    \int_{0}^{T} \cK(T - \tau) \varphi(\tau) \dd \tau = \frac{1}{N}\left(\int_{0}^{T}\sqrt{\cK(T-\tau)}\dd \tau\right)^2 \eqsim \frac{T^{\tfrac{1}{q}}}{N}. 
\end{equation*}

Consequently, the total excess risk is given by $$\cE_{N, T} \eqsim T^{-s} + \frac{ T^{\tfrac{1}{q}}}{N}.$$ Minimizing this risk with respect to $T$, we obtain the optimal intrinsic time horizon 
\[
T_{\text{opt}} \eqsim N^{\frac{q}{1+sq}}.
\]

\subsection{Step 3: Incorporating the peak learning rate constraint}

Based on the unconstrained solution, the implied peak learning rate scales as:
\begin{equation*}
    \eta(0) \eqsim N^{-\frac{1+q(s-1)}{sq+1}}. 
\end{equation*}
This scaling behavior reveals two distinct regimes:
\begin{itemize}
    \item \textbf{Easy-task regime ($s > 1 - 1/q$):} In this case, the exponent is negative, meaning $\eta(0) \to 0$ as $N \to \infty$. Consequently, the physical constraint $\eta_{\text{peak}} \leq \eta_{\stability}$ is naturally satisfied (inactive) for sufficiently large $N$.
    \item \textbf{Hard-task regime ($s < 1 - 1/q$):} Conversely, the optimal unconstrained peak diverges as $N \to \infty$. This violates the stability constraint $\eta(z) \leq \eta_{\stability}$. Therefore, the peak constraint becomes \textbf{active} and must be explicitly incorporated into the optimization.
\end{itemize}

To address the hard-task regime, we apply the Karush-Kuhn-Tucker (KKT) conditions to solve the fully constrained variational problem.

To explicitly incorporate the maximal learning rate constraint $t'(z) \leq \eta_{\stability}$, we formulate the Lagrangian in the training step domain. We introduce the Lagrange multipliers: \[ \lambda(z)\ge 0 \quad (\text{dual for } t'(z)\le \eta_{\stability}), \qquad \mu \in\mathbb R \quad (\text{dual for } t(N)-T=0). \] The generalized Lagrangian functional $\cJ$ is defined as: \[ \mathcal J[t,T,\lambda,\mu] \;=\; S(T) \;+\; \int_{0}^{N} \Bigl[ L\!\left(T,t,t'\right) +\lambda(z)\bigl(t'(z)-\eta_{\stability}\bigr) \Bigr]\dd z \;+\; \mu \bigl(t(N)-T\bigr). \]

\begin{theorem}[KKT condition]
    A feasible pair \((t^{\star},T^{\star})\) is optimal only if there exist multipliers 
\(\lambda^{\star}\in L^{\infty}([0,N])\) and \(\mu^{\star}\in\mathbb R\) such that
\begin{alignat*}{2}
\text{(1) Euler--Lagrange stationarity:}\quad
&\partial_{t}L
   -\frac{\mathrm d}{\mathrm dz}\!
        \Bigl(\partial_{t'}L+\lambda^{\star}\Bigr)
=0;\\[4pt]
\text{(2) Boundary stationarity:}\quad
&\Bigl[\partial_{t'}L+\lambda^{\star}\Bigr]_{z=N}
   +\mu^{\star}
=0;\\[4pt]
\text{(3) Scalar stationarity (w.r.t.\ \(T\)):}\quad
&S'(T^{\star})
  +\int_{0}^{N}\partial_{T}L\,\mathrm dz
  -\mu^{\star}
=0;\\[6pt]
\text{(4) Primal feasibility:}\quad
&t^{\star}(0)=0,\quad
t^{\star}(N)=T^{\star},\quad
t^{\star\,\prime}(z)\le \eta_{\stability};\\[4pt]
\text{(5) Dual feasibility:}\quad
&\lambda^{\star}(z)\ge 0; \\[4pt]
\text{(6) Complementary slackness:}\quad
&\lambda^{\star}(z)\,\bigl(t^{\star\,\prime}(z)-\eta_{\stability}\bigr)=0.
\end{alignat*}
\end{theorem}

If \(t^{\star\,\prime}(z)<\eta_{\stability}\) at some point, condition (6) forces \(\lambda^{\star}(z)=0\);
where the derivative saturates the bound (\(t^{\star\,\prime}=\eta_{\stability}\)), 
\(\lambda^{\star}\) may be positive.

Direct analysis of the above KKT condition is still complicated. A key observation that makes the analysis easier is the following observation:
\begin{proposition}[Monotonicity]
Let $t^*$ be the solution of \eqref{equ: functional_physical}. Then, $z\mapsto t'(z)$ must be non-increasing.
\end{proposition}
\begin{proof}
To prove this, we only need to show for any fixed $T$, the derivative $(t^*_T)^{\prime}$ is non-increasing. Hence, consider
\begin{equation*}
\begin{aligned}
    G_T(t)&\coloneqq \int_{0}^{N}
\bigl(1+T-t(z)\bigr)^{-(2-1/q)}
\,\bigl(t'(z)\bigr)^{2} \dd z 
\\
&= \int_0^T (1+T-\tau)^{-(2-\frac{1}{q})} \varphi(\tau) \dd \tau
\\
&\eqqcolon \int_0^T w(\tau)\varphi(\tau)\dd \tau, 
\end{aligned}
\end{equation*}
where the weight function $w(\cdot)$ is increasing.

A simple ``bubble-sort'' argument suffices to show that the optimal $\varphi$ is non-increasing. Suppose a feasible $\varphi$ is not non-increasing. There must exist $\tau_1<\tau_2$ such that $\varphi(\tau_1)<\varphi(\tau_2)$. Then, we can construct $\tilde{\varphi}$ by swapping the values on small intervals around $\tau_1, \tau_2$.
\begin{itemize}
 \item The constraint $\int_0^T\frac{1}{\tilde{\varphi}(\tau)}\dd \tau = N$ and boundedness constraint are unaffected by the swaps.
 \item The change of integral:
 \[
   \Delta=[\varphi(\tau_2)-\varphi(\tau_1)] (w(\tau_1)- w(\tau_2))<0.
 \] 
\end{itemize}
Thus the swap strictly lowers the objective. Repeating finitely many swaps (or taking a limit) yields a decreasing function with no larger cost, contradicting optimality. Hence a minimiser must be decreasing.
\end{proof}
\begin{remark}
 The above bubble-sort argument essentially adopts the (anti)-Hardy-Littlewood  inequality. 
\end{remark}

\begin{theorem}[Stable-decay shape]
\label{thm: stable-decay}
Assume $s<1-1/q$. Then the optimal LR schedule has a stable-decay structure.
More precisely, there exist $N_1\in[0,N]$ and $a\in[0,1]$ such that, letting
$T_1=t(N_1)$ and $T=t(N)$, the intrinsic-time learning rate
$\varphi(\tau)=t'(t^{-1}(\tau))$ satisfies
\begin{equation}
    \label{eq:wsd-intrinsic-time}
  \varphi(\tau)
  =
  \begin{cases}
    \eta_{\stability},
    & 0\leq \tau\leq T_1,\\[1mm]
    a\eta_{\stability}
    \left(
      \dfrac{1+T-\tau}{1+T-T_1}
    \right)^{1-\frac{1}{2q}},
    & T_1<\tau\leq T.
  \end{cases}
\end{equation}
Equivalently, in training steps, the LR schedule can be expressed as
\begin{align}\label{eqn: wsd-ansatz}
	t'(z) = 
	\begin{cases}
		\eta_{\stability} & \text{ if } 0\leq z\leq N_1\\ 
		a\eta_{\stability}\left(1+o_N(1)-\frac{z-N_1}{N-N_1}\right)^{2q-1} & \text{ if } N_1< z\leq N.
	\end{cases}
\end{align}
\end{theorem}

\begin{proof}
By the monotonicity proposition, the optimal LR schedule is non-increasing. Hence the schedule must have a stable-decay structure: there exists
$N_1\in[0,N]$ such that
\[
\begin{cases}
    t'(z)=\eta_{\stability}, \quad & z\in [0,N_1],\\
    t'(z)<\eta_{\stability}, \quad & z\in (N_1,N].
\end{cases}
\]
In the decay phase,
complementary slackness gives $\lambda(z)=0$, so the Euler--Lagrange equation
reduces to the same unconstrained profile equation as in Step 1, restricted to
the interval $[T_1,T]$ in intrinsic time. By the derivation of \eqref{equ: optimal_lrs_intrinsic}, we have
\[
  \varphi(\tau)
  =
  a\eta_{\stability}
  \left(
    \frac{1+T-\tau}{1+T-T_1}
  \right)^{1-\frac{1}{2q}},
  \qquad T_1<\tau\leq T,
\]
where $a\eta_{\stability}$ is the learning rate at the beginning of the decay phase (we have $a\in[0,1]$ due to monotonicity).
Together with $\varphi(\tau)=\eta_{\stability}$ on $[0,T_1]$, this gives the
intrinsic-time representation \eqref{eq:wsd-intrinsic-time}.

Applying the same calculation as in Step~1 to
$\dd\tau/\dd z=\varphi(\tau)$ on the interval $[N_1,N]$, we obtain
\[
t'(z)
=
a\eta_{\stability}
\left(
1+o_N(1)-\frac{z-N_1}{N-N_1}
\right)^{2q-1}.
\]
This proves the theorem.

\end{proof}

\begin{theorem}[Optimal stable-decay parameters]
Suppose $s<1-1/q$.
Let $t_{a,r}$ be the intrinsic-time map determined by the exact
intrinsic-time profile \eqref{eq:wsd-intrinsic-time} and the training
horizon $N$, where $r=(N-N_1)/N$ is the ratio of the decay phase.
Let $T_{a,r}\coloneqq t_{a,r}(N)$ and define
\[
    \cQ_N(a,r)
    \coloneqq
    \widetilde\cF[t_{a,r},T_{a,r}].
\]
Then, for sufficiently large $N$, $\cQ_N$ attains its minimum over
$(0,1]\times(0,1)$.
Every minimizer
\[
    (a_N^*,r_N^*)
    \in
    \argmin_{a\in(0,1],\,r\in(0,1)}
    \cQ_N(a,r)
\]
satisfies
\[
    a_N^*=1,
    \qquad
    r_N^*\eqsim N^{-\kappa},
    \qquad
    \kappa=\frac{1-1/q-s}{2-1/q}.
\]
\end{theorem}

\begin{proof}
For notational simplicity, we assume $\eta_{\stability}=1$
throughout the proof.
By Theorem~\ref{thm: stable-decay}, it suffices to optimize the
stable-decay family over its peak parameter $a$ and decay ratio $r$.
Recall that $s_c=1-1/q$, so that $0<s<{s_c}<1$.

\paragraph*{Simplifying the objective function.}
Let $N_1=(1-r)N$ and $T_1=N_1$.
By the intrinsic-time representation \eqref{eq:wsd-intrinsic-time},
the LR schedule is
\[
  \varphi(\tau)=
  \begin{cases}
    1, & 0\leq \tau\leq T_1,\\[1mm]
    a\left(
        \dfrac{1+T-\tau}{1+T-T_1}
    \right)^{1-\frac{1}{2q}},
    & T_1<\tau\leq T.
  \end{cases}
\]
Define $\Delta\coloneqq T-T_1$.
The decay-step constraint gives
\[
rN
=
\int_{T_1}^{T}\frac{1}{\varphi(\tau)}\dd\tau
=
\frac{2q}{a}
\left[
    (1+\Delta)
    -(1+\Delta)^{1-\frac{1}{2q}}
\right].
\]
Thus $\Delta=\Delta(a,r)$ is determined by
\begin{equation}
    \label{eq:proof-delta-ar}
    (1+\Delta)
    -(1+\Delta)^{1-\frac{1}{2q}}
    =
    \frac{arN}{2q}.
\end{equation}
The noise term is
\begin{align*}
\int_0^N L(T,t,t')\dd z
&=
\int_0^T
    (1+T-\tau)^{-(2-\frac{1}{q})}
    \varphi(\tau)\dd\tau\\
&=
\int_0^{T_1}
    (1+T-\tau)^{-(2-\frac{1}{q})}\dd\tau\\
&\quad
+a\int_{T_1}^T
    (1+T-\tau)^{-(2-\frac{1}{q})}
    \left(
        \frac{1+T-\tau}{1+\Delta}
    \right)^{1-\frac{1}{2q}}
    \dd\tau\\
&=
\frac{1}{{s_c}}
\left[
    (1+\Delta)^{-{s_c}}
    -(1+T)^{-{s_c}}
\right]
+a(1+\Delta)^{-{s_c}}
\int_0^{\frac{\Delta}{1+\Delta}}
    (1-u)^{-1+\frac{1}{2q}}\dd u\\
&=
\frac{1}{{s_c}}
\left[
    (1+\Delta)^{-{s_c}}
    -(1+T)^{-{s_c}}
\right]\\
&\quad
+2q a
\left[
    (1+\Delta)^{-{s_c}}
    -(1+\Delta)^{-(1-\frac{1}{2q})}
\right].
\end{align*}
Combining this with the signal-learning term $(1+T)^{-s}$,
the objective is
\begin{align*}
G(a,r)
\coloneqq\;&
(1+T)^{-s}
+\frac{1}{{s_c}}
\left[
    (1+\Delta)^{-{s_c}}
    -(1+T)^{-{s_c}}
\right]\\
&+
2q a
\left[
    (1+\Delta)^{-{s_c}}
    -(1+\Delta)^{-(1-\frac{1}{2q})}
\right],
\end{align*}
where $\Delta=\Delta(a,r)$ and
$T=(1-r)N+\Delta(a,r)$.
Thus $G(a,r)=\cQ_N(a,r)$ on the original parameter domain.

\paragraph*{Existence and exclusion of boundary minimizers.}
Define
\[
    \psi(\Delta)
    \coloneqq
    (1+\Delta)
    -(1+\Delta)^{1-\frac{1}{2q}},
    \qquad \Delta\geq0.
\]
Since $\psi(0)=0$ and
\[
    \psi'(\Delta)
    =
    1-\left(1-\frac{1}{2q}\right)
      (1+\Delta)^{-\frac{1}{2q}}
    \geq\frac{1}{2q}>0,
\]
equation \eqref{eq:proof-delta-ar} defines a unique continuous
function $\Delta(a,r)$ on $[0,1]^2$.
Moreover,
\[
    0\leq\Delta(a,r)\leq arN,
    \qquad
    0\leq T=(1-r)N+\Delta(a,r)\leq N.
\]
Consequently, the formula for $G$ extends continuously to
$[0,1]^2$, and this extension attains its minimum.

We first construct a comparison point.
For $(a,r)=(1,1/2)$, equation \eqref{eq:proof-delta-ar} yields
\[
    \frac{\Delta}{N}\longrightarrow\frac{1}{4q},
    \qquad
    \frac{T}{N}\longrightarrow c_q,
    \qquad
    c_q\coloneqq\frac12+\frac{1}{4q}.
\]
The noise term at this point is $O(N^{-{s_c}})$.
Since $s<{s_c}$, it follows that
\[
    G(1,1/2)
    =
    \left(c_q^{-s}+o_N(1)\right)N^{-s}.
\]
In particular, the minimum of $G$ tends to zero.

On either boundary $a=0$ or $r=0$, we have $\Delta=0$, and hence
\[
    G(a,r)
    =
    (1+T)^{-s}
    +\frac{1}{{s_c}}
       \left[1-(1+T)^{-{s_c}}\right]
    \geq1.
\]
Here we used $(1+T)^{-s}\geq(1+T)^{-{s_c}}$ and ${s_c}<1$.
Thus neither boundary can contain a minimizer for sufficiently
large $N$.

It remains to exclude $r=1$.
On this boundary,
\[
    T=\Delta(a,1)\leq\Delta(1,1),
    \qquad
    \frac{\Delta(1,1)}{N}
    \longrightarrow\frac{1}{2q}.
\]
Since both noise contributions in $G$ are nonnegative, uniformly
over $a\in[0,1]$ we have
\[
    G(a,1)
    \geq
    (1+\Delta(1,1))^{-s}
    =
    \left((2q)^s+o_N(1)\right)N^{-s}.
\]
Because $c_q>1/(2q)$, we have
$c_q^{-s}<(2q)^s$.
Therefore, for sufficiently large $N$, every point on $r=1$
has a larger objective value than the comparison point $(1,1/2)$.

We conclude that every minimizer of the continuous extension
satisfies
\[
    a_N^*>0,
    \qquad
    0<r_N^*<1.
\]
In particular, the minimum over the original parameter domain
$(0,1]\times(0,1)$ is attained.

\paragraph*{The optimal peak learning rate $a_N^*$.}
Differentiating \eqref{eq:proof-delta-ar} gives
\[
\frac{\partial\Delta}{\partial a}
=
\frac{rN}{
2q\left[
1-\left(1-\frac{1}{2q}\right)
(1+\Delta)^{-\frac{1}{2q}}
\right]},
\qquad
\frac{\partial\Delta}{\partial r}
=
\frac{aN}{
2q\left[
1-\left(1-\frac{1}{2q}\right)
(1+\Delta)^{-\frac{1}{2q}}
\right]}.
\]
In particular,
\[
\frac{\partial\Delta/\partial a}
     {\partial\Delta/\partial r}
=
\frac{r}{a}.
\]

Taking full partial derivatives of $G(a,r)$ gives
\begin{align*}
\frac{\partial G}{\partial a}
&=
2q
\left[
    (1+\Delta)^{-{s_c}}
    -(1+\Delta)^{-(1-\frac{1}{2q})}
\right]
+
\frac{\partial\Delta}{\partial a}\,A,\\
\frac{\partial G}{\partial r}
&=
N
\left[
    s(1+T)^{-s-1}
    -(1+T)^{-{s_c}-1}
\right]
+
\frac{\partial\Delta}{\partial r}\,A,
\end{align*}
where
\begin{align*}
A
&=
-s(1+T)^{-s-1}
+(1+T)^{-{s_c}-1}
-(1+\Delta)^{-{s_c}-1}\\
&\quad
+2q a
\left[
    -{s_c}(1+\Delta)^{-{s_c}-1}
    +\left(1-\frac{1}{2q}\right)
      (1+\Delta)^{-(2-\frac{1}{2q})}
\right].
\end{align*}

Fix a minimizer $(a_N^*,r_N^*)$.
Since $0<r_N^*<1$, we have $\partial_rG=0$ at this point.
All identities obtained from this stationarity condition below
are evaluated at this minimizer.
Using the ratio above, we eliminate $A$ from $\partial_aG$ and obtain
\begin{equation}
\begin{aligned}
\frac{\partial G}{\partial a}
&=
2q
\left[
    (1+\Delta)^{-{s_c}}
    -(1+\Delta)^{-(1-\frac{1}{2q})}
\right]\\
&\quad
-\frac{rN}{a}
\left[
    s(1+T)^{-s-1}
    -(1+T)^{-{s_c}-1}
\right].
\end{aligned}
\label{eq:proof-Ga-after-Gr}
\end{equation}
The same stationarity condition gives
\begin{equation}
\begin{aligned}
&s(1+T)^{-s-1}
-(1+T)^{-{s_c}-1}\\
&\qquad=
\frac{
a(1+\Delta)^{-{s_c}-1}
\left[
    1+2q a{s_c}
    -2q a\left(1-\frac{1}{2q}\right)
    (1+\Delta)^{-\frac{1}{2q}}
\right]
}{
2q\left[
    1-\left(1-\frac{1}{2q}\right)
    (1+\Delta)^{-\frac{1}{2q}}
\right]-a
}.
\end{aligned}
\label{eq:proof-stationary-bracket}
\end{equation}
Substituting \eqref{eq:proof-stationary-bracket} into
\eqref{eq:proof-Ga-after-Gr}, and using
\eqref{eq:proof-delta-ar}, yields
\[
\frac{\partial G}{\partial a}
=
-
\frac{
    2q(1+\Delta)^{-{s_c}}
    \left[
        1-(1+\Delta)^{-\frac{1}{2q}}
    \right]
    (1-a)^2
}{
    a\left(
        2q\left[
            1-\left(1-\frac{1}{2q}\right)
            (1+\Delta)^{-\frac{1}{2q}}
        \right]-a
    \right)
}.
\]
The denominator is strictly positive for $a\in(0,1]$ and
$\Delta>0$, since
\[
    2q\left[
        1-\left(1-\frac{1}{2q}\right)
        (1+\Delta)^{-\frac{1}{2q}}
    \right]>1.
\]

If $a_N^*<1$, the minimizer is interior in $a$, so
$\partial_aG(a_N^*,r_N^*)=0$.
However, the preceding identity gives
$\partial_aG(a_N^*,r_N^*)<0$, a contradiction.
Therefore every minimizer satisfies
\[
    a_N^*=1.
\]

\paragraph*{The optimal decay duration $r_N^*$.}
Set $a=1$ and write $\Delta(r)=\Delta(1,r)$ and
$T(r)=(1-r)N+\Delta(r)$.
Then
\[
    \frac{rN}{2q}
    =
    (1+\Delta)
    -(1+\Delta)^{1-\frac{1}{2q}},
\]
and
\[
\Delta'(r)
=
\frac{N}{
2q\left[
    1-\left(1-\frac{1}{2q}\right)
    (1+\Delta)^{-\frac{1}{2q}}
\right]},
\qquad
0<\Delta'(r)<N.
\]
Along this constraint, direct differentiation gives
\[
\frac{\dd}{\dd r}G(1,r)
=
\bigl(N-\Delta'(r)\bigr)
\left[
    s(1+T)^{-s-1}
    -(1+T)^{-{s_c}-1}
    -(1+\Delta)^{-{s_c}-1}
\right].
\]
Since $r_N^*$ is interior and $N-\Delta'(r_N^*)>0$,
the stationarity condition at the minimizer is
\[
    s(1+T)^{-s-1}
    -(1+T)^{-{s_c}-1}
    =
    (1+\Delta)^{-{s_c}-1}.
\]

We next justify the large-$N$ asymptotics in this identity.
The comparison point constructed above gives
\[
    (1+T)^{-s}
    \leq G(1,r_N^*)
    \leq G(1,1/2)
    \lesssim N^{-s}.
\]
Together with $T\leq N$, this implies $T\eqsim N$.
Moreover, the nonnegativity of the other terms in $G$ gives
\[
    (1+\Delta)^{-{s_c}}
    \leq
    (1+T)^{-{s_c}}
    +{s_c} G(1,r_N^*)
    \longrightarrow0.
\]
Hence $\Delta\to\infty$ as well.

Since $s<{s_c}$, the term $(1+T)^{-{s_c}-1}$ is negligible
relative to $s(1+T)^{-s-1}$.
The stationarity condition therefore yields
\[
    1+\Delta
    \eqsim
    (1+T)^{\frac{s+1}{{s_c}+1}}.
\]
In particular, $\Delta=o(T)$.
Using \eqref{eq:proof-delta-ar} with $a=1$ and $\Delta\to\infty$,
we obtain
\[
    r_N^*N\eqsim\Delta.
\]
Combining this with $T\eqsim N$ gives
\[
    r_N^*
    \eqsim
    N^{\frac{s+1}{{s_c}+1}-1}
    =
    N^{-\frac{{s_c}-s}{{s_c}+1}}
    =
    N^{-\frac{1-1/q-s}{2-1/q}}.
\]
This completes the proof.
\end{proof}

\section{Proofs for fractional LR schedules}
\label{app:fixed-profile-proofs-intrinsic}

\paragraph*{Tail exponent under intrinsic-time}\leavevmode
The power-decay tail in~\eqref{eqn:power-tail} is stated in training steps:
there exist constants $p\ge0$ and $\delta\in(0,1)$ such that
\[
    \zeta(x)\eqsim(1-x)^p,\qquad x\in[\delta,1).
\]
The next lemma identifies the corresponding tail behavior of {$\psi$}.

\begin{lemma}[Intrinsic-time tail exponent]\label{lem:tildezeta-tail}
Assume $\zeta(x)\eqsim(1-x)^p$ for all $x\in[\delta,1)$ with some $p\ge0$.
{Let $\widetilde\delta\coloneqq P(\delta)\in[0,1)$.} Then for all $y\in[\widetilde\delta,1)$,
\[
    \psi(y)\eqsim(1-y)^{\frac{p}{p+1}}.
\]
\end{lemma}

\begin{proof}
Fix $y\in[\widetilde\delta,1)$ and define $x\coloneqq P^{-1}(y)\in[\delta,1)$. Then $\psi(y)=\zeta(x)$ by definition. Let $A(x)\coloneqq\int_x^1\zeta(u)\dd u=1-P(x)$.
Since $\zeta(u)\eqsim(1-u)^p$ on $[\delta,1)$, integrating yields
\[
    A(x)\eqsim(1-x)^{p+1},\qquad x\in[\delta,1).
\]
Moreover, $1-y=1-P(x)=A(x)\eqsim(1-x)^{p+1}$. Hence, for $y\in[\widetilde\delta,1)$,
\[
    1-x\eqsim(1-y)^{\frac{1}{p+1}}.
\]
Finally, using $\zeta(x)\eqsim(1-x)^p$ on $[\delta,1)$ and $\psi(y)=\zeta(x)$, we obtain
\[
    \psi(y)=\zeta(x)\eqsim(1-x)^p
    \eqsim\Bigl((1-y)^{\frac{1}{p+1}}\Bigr)^p
    =(1-y)^{\frac{p}{p+1}},
\]
which holds for all $y\in[\widetilde\delta,1)$.
\end{proof}

\subsection{Proof of Theorem~\ref{thm:fractional-fsl}}
\label{app:proof-thm-fractional-fsl}

\begin{proof}
Recall $\cK(u)=(1+u)^{-\rho}$, where $\rho\coloneqq2-1/q\in(1,2)$. Set $a\coloneqq p/(p+1)\in[0,1)$. By \eqref{eqn: intrinsic-time LR},  $\varphi(t)=(T/N)\psi(t/T)$. By Lemma~\ref{lem:tildezeta-tail}, the bounded profile $\psi$ satisfies $\psi(y)\eqsim(1-y)^a$ on $[\widetilde\delta,1)$ for some $\widetilde\delta\in[0,1)$.

Recall
\begin{equation}\label{eq:fsl-TN}
    \cF_N(T,\zeta)=(1+T)^{-s}
    +\frac{T}{N}\int_0^T(1+T-\tau)^{-\rho}\psi(\tau/T)\dd\tau.
\end{equation}
It remains to estimate the noise term. Lemma~\ref{lem:rescaled-profile-convolution} below gives
\[
    \frac{T}{N}\int_0^T(1+T-\tau)^{-\rho}\psi(\tau/T)\dd\tau
    \eqsim\frac{T^{1/\alpha}}{N}
    \bigl[{\log T}\bigr]^{\mathbf{1}\{q=p+1\}},
\]
where $\alpha=\min\{q,p+1\}$, since $\min\{a,\rho-1\}=1-1/\alpha$ and $a=\rho-1$ exactly when $q=p+1$. Substituting this estimate into~\eqref{eq:fsl-TN} proves the scaling law.
\end{proof}

\begin{lemma}[Convolution with a rescaled decay profile]\label{lem:rescaled-profile-convolution}
Let $g:[0,1]\to\RR_{\ge0}$ be a bounded function. Suppose $g(x)\eqsim(1-x)^a$ for $x\in[\delta,1)$, where $a\ge0$ and $\delta\in[0,1)$. For every $\rho>1$ and ${ T\gtrsim1}$,
\begin{equation}\label{eq:rescaled-profile-convolution}
    \int_0^T(1+T-t)^{-\rho}g(t/T)\dd t
    \eqsim
    \begin{cases}
        T^{-a}, & a<\rho-1,\\
        T^{1-\rho}{\log T}, & a=\rho-1,\\
        T^{1-\rho}, & a>\rho-1.
    \end{cases}
\end{equation}
The comparison constants may depend on $g,a,\rho,\delta$, but are independent of $T$.
\end{lemma}

\begin{proof}
Changing variables $t=Tx$ gives
\[
    \int_0^T(1+T-t)^{-\rho}g(t/T)\dd t
    =T\int_0^1\bigl(1+T(1-x)\bigr)^{-\rho}g(x)\dd x.
\]
We split the integral into an early region and a tail region.
Denote the integrals over $[0,\delta]$ and $[\delta,1]$ by $I_{\mathrm{early}}$ and $I_{\mathrm{tail}}$, respectively.

\paragraph{(i) Early region.}\leavevmode
For $x\in[0,\delta]$, we have $1-x\ge1-\delta$, hence
$\bigl(1+T(1-x)\bigr)^{-\rho}\eqsim T^{-\rho}$ uniformly for ${ T\gtrsim1}$.
Therefore,
\[
    I_{\mathrm{early}}\lesssim T^{-\rho}\int_0^\delta g(x)\dd x
    \lesssim T^{-\rho}.
\]

\paragraph{(ii) Tail region and the logarithmic boundary.}\leavevmode
For $x\in[\delta,1)$, the tail assumption gives $g(x)\eqsim(1-x)^a$, hence
\[
    I_{\mathrm{tail}}\eqsim
    \int_\delta^1\bigl(1+T(1-x)\bigr)^{-\rho}(1-x)^a\dd x.
\]
Let $u=T(1-x)$, so that $\dd x=-\dd u/T$ and $1-x=u/T$. Then
\begin{equation}\label{eq:noise-tail-master}
    T I_{\mathrm{tail}}\eqsim
    T^{-a}\int_0^{T(1-\delta)}(1+u)^{-\rho}u^a\dd u.
\end{equation}
Suppose ${ T\ge2/(1-\delta)}$. We split this integral at $u=1$. The part over $[0,1]$ is a positive constant, while $(1+u)^{-\rho}u^a\eqsim u^{a-\rho}$ for $u\ge1$. We distinguish three cases.

\begin{itemize}
\item  Case 1: $a-\rho<-1$. The integral in~\eqref{eq:noise-tail-master} is bounded above and below by positive constants, and hence $T I_{\mathrm{tail}}\eqsim T^{-a}$.
\item  Case 2: $a-\rho>-1$. The integral in~\eqref{eq:noise-tail-master} is comparable to $T^{a+1-\rho}$, so $T I_{\mathrm{tail}}\eqsim T^{1-\rho}$.
\item  Case 3: $a-\rho=-1$. Integrating $u^{-1}$ gives a factor comparable to ${\log T}$, so $T I_{\mathrm{tail}}\eqsim T^{1-\rho}{\log T}$.
\end{itemize}
 In all three cases, the early contribution $T I_{\mathrm{early}}\lesssim T^{1-\rho}$ is bounded by the tail contribution. Combining the two regions proves~\eqref{eq:rescaled-profile-convolution}.
\end{proof}

\subsection{Proof of Theorem~\ref{thm: optimal-fractional-lrs}}
\label{app:proof-thm-opt-fractional}

\begin{proof}
 Recall $\alpha=\min\{q,p+1\}$. By Theorem~\ref{thm:fractional-fsl}, for ${ T\gtrsim1}$ the final-step loss satisfies
\begin{equation}\label{eq:opt-frac-start}
    \cF_N(T,\zeta)\eqsim T^{-s}
    +\frac{T^{1/\alpha}}{N}\bigl[{\log T}\bigr]^{\mathbf{1}\{q=p+1\}}.
\end{equation}
We minimize over the admissible range $0<T\le c_\zeta N$, where $c_\zeta\coloneqq\eta_{\stability}/\|\zeta\|_\infty$ is a fixed positive constant. { It suffices to consider $T\gtrsim1$.}

\paragraph{(i) Hard-task regime: $s<1-\frac{1}{\alpha}$.}
We first show that in the hard regime, the noise term is uniformly negligible. Since $T\le c_\zeta N$, uniformly for ${ 2\le T\le c_\zeta N}$,
\[
    \frac{T^{1/\alpha}}{N}\bigl[{\log T}\bigr]^{\mathbf{1}\{q=p+1\}}
    \lesssim N^{-(1-1/\alpha)}\bigl[{\log N}\bigr]^{\mathbf{1}\{q=p+1\}}
    =o(N^{-s}).
\]
Since $T^{-s}\gtrsim N^{-s}$ throughout this range, the loss satisfies $\cF_N(T,\zeta)\eqsim T^{-s}$. At the admissible choice $T=c_\zeta N$, the loss is therefore comparable to $N^{-s}$. Conversely, the signal term gives $\cF_N(T,\zeta)\gtrsim N^{-s}$ for every admissible $T$. Hence $\cE_N^*\eqsim N^{-s}$. Finally, $(T^*)^{-s}\lesssim\cE_N^*\eqsim N^{-s}$ gives $T^*\gtrsim N$, and the stability constraint gives $T^*\lesssim N$. Thus $T^*\eqsim N$. This proves the hard-task statement, including the boundary case $q=p+1$, where the extra logarithmic factor is dominated by the strict power gap.

\paragraph{(ii) Easy-task regime: $s\ge1-\frac{1}{\alpha}$.}
In this regime, the optimal choice balances the two terms in
{ \eqref{eq:opt-frac-start}}. We consider two cases.

\begin{itemize}
\item Case 1: $q\neq p+1$. Then the logarithmic factor is absent and~\eqref{eq:opt-frac-start} becomes
\[
    \cF_N(T,\zeta)\eqsim T^{-s}+\frac{T^{1/\alpha}}{N}.
\]
Balancing the two terms gives $T^{s+1/\alpha}\eqsim N$. Set $\widehat T=N^{1/(s+1/\alpha)}$.
Since $s+1/\alpha\ge1$, a fixed positive multiple of $\widehat T$ is admissible, and substituting it into either term yields
\[
    \cE_N^*\lesssim\widehat T^{-s}=N^{-\frac{s\alpha}{s\alpha+1}}.
\]
For $T\le\widehat T$, the signal term is at least $\widehat T^{-s}$; for $T\ge\widehat T$, the noise term is at least $\widehat T^{1/\alpha}/N=\widehat T^{-s}$.
This proves the matching lower bound. Moreover, bounding each term at $T^*$ by a constant times $\widehat T^{-s}$ gives $T^*\gtrsim\widehat T$ from the signal term and $T^*\lesssim\widehat T$ from the noise term. Hence
\[
    T^*\eqsim N^{\frac{1}{s+1/\alpha}}.
\]

\item Case 2: $q=p+1$. Then $\alpha=q$ and the second term carries the logarithmic factor:
\[
    \cF_N(T,\zeta)\eqsim T^{-s}+\frac{T^{1/\alpha}}{N}{\log T}.
\]
Let ${\widehat T>1}$ be the solution of the balance equation
\begin{equation}\label{eq:balance-boundary}
    \widehat T^{s+1/\alpha}{\log\widehat T}=N.
\end{equation}
Taking logarithms gives
\[
    (s+1/\alpha)\log\widehat T+\log{\log\widehat T}=\log N.
\]
Since $\widehat T\to\infty$ and $\log{\log\widehat T}=o(\log\widehat T)$, we have ${\log\widehat T}\eqsim\log N$.
Substituting this back into~\eqref{eq:balance-boundary} yields
\[
    \widehat T\eqsim\left(\frac{N}{\log N}\right)^{\frac{1}{s+1/\alpha}}.
\]
In particular, $\widehat T=o(N)$, so it is admissible for all sufficiently large $N$.
The same comparison of the decreasing signal term and increasing noise term gives $\cE_N^*\eqsim\widehat T^{-s}$ and $T^*\eqsim\widehat T$.
Substituting into the signal term gives
\[
    \cE_N^*\eqsim N^{-\frac{s\alpha}{s\alpha+1}}
    (\log N)^{\frac{s\alpha}{s\alpha+1}}.
\]
\end{itemize}
\end{proof}

\subsection{Proof of Theorem~\ref{thm:horizon-free-lower-bound}}
\label{app:proof-horizon-free-lower-bound}

\paragraph{Proof.}
Let
\[
    \eta_N(z)
    =
    \eta_{0}{h}(z),
    \qquad
    {H}(x)
    =
    \int_0^x{h}(z)\dd z,
\]
and recall that
\[
    T_N
    =
    \eta_{0}{H}(N).
\]
We prove the result by contradiction. Suppose that there exist
\[
    \rho>\frac{s}{s+1},
    \qquad
    C>0,
    \qquad
    N_0\in\NN_+,
\]
such that
\begin{equation}
\label{eqn:proof-horizon-free-assumption}
    \cF[\eta_N]
    \leq
    C N^{-\rho},
    \qquad
    \forall\,N\geq N_0.
\end{equation}

We first derive a lower bound from the signal-learning term. By the definition
of the FSL functional,
\[
    \cF[\eta_N]
    \geq
    (1+T_N)^{-s}.
\]
Combining this with
\eqref{eqn:proof-horizon-free-assumption} gives
\[
    (1+T_N)^{-s}
    \leq
    C N^{-\rho},
\]
and hence, for all sufficiently large $N$,
\begin{equation}
\label{eqn:proof-horizon-free-T-lb}
    T_N
    \gtrsim
    N^{\rho/s}.
\end{equation}
In particular, $T_N\geq 1$ for all sufficiently large $N$.

Next, consider the noise-forgetting term
\[
    Q_N
    \coloneqq
    \int_0^{T_N}
    \cK(T_N-\tau)\varphi_N(\tau)\dd\tau,
\]
where
\[
    \varphi_N(\tau)
    =
    \eta_N\bigl(t_N^{-1}(\tau)\bigr).
\]
Since ${h}$ is non-increasing, $\eta_N$ is non-increasing in training steps.
Since $t_N$ is increasing, $\varphi_N$ is also non-increasing in intrinsic
time. Therefore, for $T_N\geq 1$,
\begin{align}
    Q_N
    &\geq
    \int_{T_N-1}^{T_N}
    \cK(T_N-\tau)\varphi_N(\tau)\dd\tau
    \notag\\
    &\geq
    \varphi_N(T_N)
    \int_{T_N-1}^{T_N}
    \cK(T_N-\tau)\dd\tau
    \notag\\
    &=
    \eta_N(N)
    \int_0^1\cK(u)\dd u.
\label{eqn:proof-horizon-free-Q-lb}
\end{align}
Since
\[
    \cK(u)
    =
    (1+u)^{-(2-1/q)}
\]
is strictly positive on $[0,1]$, the constant
\[
    c_{\cK}
    \coloneqq
    \int_0^1\cK(u)\dd u
\]
is positive and independent of $N$. Hence
\begin{equation}
\label{eqn:proof-horizon-free-terminal-lb}
    Q_N
    \gtrsim
    \eta_N(N)
    =
    \eta_{0}{h}(N).
\end{equation}
Combining
\eqref{eqn:proof-horizon-free-assumption} and
\eqref{eqn:proof-horizon-free-terminal-lb} yields
\begin{equation}
\label{eqn:proof-horizon-free-terminal-ub}
    \eta_{0}{h}(N)
    \lesssim
    N^{-\rho}.
\end{equation}

Using
\[
    T_N
    =
    \eta_{0}{H}(N),
\]
together with
\eqref{eqn:proof-horizon-free-T-lb} and
\eqref{eqn:proof-horizon-free-terminal-ub}, we obtain
\begin{align}
    \frac{{h}(N)}{{H}(N)}
    &=
    \frac{\eta_{0}{h}(N)}
         {\eta_{0}{H}(N)}
    =
    \frac{\eta_{0}{h}(N)}{T_N}
    \notag\\
    &\lesssim
    N^{-\rho}
    N^{-\rho/s}
    =
    N^{-\rho(1+1/s)}.
\label{eqn:proof-horizon-free-log-bound-integers}
\end{align}
Set
\[
    \nu
    \coloneqq
    \rho\left(1+\frac{1}{s}\right).
\]
Because $\rho>s/(s+1)$, we have $\nu>1$.

The inequality above is stated at integer values of $N$. We next extend it to
all sufficiently large real $x$. Let $N=\lfloor x\rfloor$. Since ${h}$ is
non-increasing and ${H}$ is non-decreasing,
\[
    {h}(x)
    \leq
    {h}(N),
    \qquad
    {H}(x)
    \geq
    {H}(N).
\]
Consequently,
\[
    \frac{{h}(x)}{{H}(x)}
    \leq
    \frac{{h}(N)}{{H}(N)}
    \lesssim
    N^{-\nu}
    \lesssim
    x^{-\nu}.
\]
Since ${H}$ is absolutely continuous and
${H}'(x)={h}(x)$ almost everywhere,
\begin{equation}
\label{eqn:proof-horizon-free-log-derivative}
    \frac{\dd}{\dd x}\log{H}(x)
    =
    \frac{{h}(x)}{{H}(x)}
    \lesssim
    x^{-\nu}
\end{equation}
for almost every sufficiently large $x$.

Integrating
\eqref{eqn:proof-horizon-free-log-derivative} from a sufficiently large
constant $x_0$ to $R$ gives
\begin{align}
    \log{H}(R)-\log{H}(x_0)
    &\lesssim
    \int_{x_0}^R x^{-\nu}\dd x
    \notag\\
    &\leq
    \int_{x_0}^{\infty}x^{-\nu}\dd x
    <
    \infty,
\end{align}
where the final integral is finite because $\nu>1$. It follows that
\begin{equation}
\label{eqn:proof-horizon-free-Psi-bounded}
    \sup_{R\geq x_0}{H}(R)
    <
    \infty.
\end{equation}

On the other hand, {the uniform bound $\eta_0\lesssim1$} and
\eqref{eqn:proof-horizon-free-T-lb} imply
\begin{equation}
\label{eqn:proof-horizon-free-Psi-lb}
    {H}(N)
    =
    \frac{T_N}{\eta_{0}}
    \mathrel{\gtrsim}
    {T_N}
    \gtrsim
    N^{\rho/s}.
\end{equation}
Thus ${H}(N)\to\infty$, contradicting
\eqref{eqn:proof-horizon-free-Psi-bounded}. Therefore,
\eqref{eqn:proof-horizon-free-assumption} cannot hold for any
$\rho>s/(s+1)$, proving
\[
    \limsup_{N\to\infty}
    N^\rho\cF[\eta_N]
    =
    \infty.
\]
\qed

\section{Proofs for discrete-time SGD}

We prove Theorem~\ref{thm:power-rate} and Proposition~\ref{prop:lower-bound-any-lrs} under the source and capacity conditions in Assumptions~\ref{ass: source} and~\ref{ass: capacity}. These conditions specify an $\ell^2$ bound on the spectral source coefficients and a one-sided upper bound on the eigenvalues. They yield the signal-learning and noise-forgetting exponents used in the FSL analysis.

Section~\ref{subsec:setup} introduces the kernel regression setup. Section~\ref{app:Analysis of SGD dynamics} derives the SGD risk inequality and identifies the functions $\cS$ and $\cK$. Section~\ref{subsec:bounding} establishes their properties, and Section~\ref{subsec:risk-from-properties} uses these properties to bound the excess risk. Sections~\ref{app:proof-theorem-power-cosine} and~\ref{app:proof-prop-lower-bound-any-lrs} then prove Theorem~\ref{thm:power-rate} and Proposition~\ref{prop:lower-bound-any-lrs}, respectively.

\subsection{Setup}
\label{subsec:setup}

We follow the setup of Section~\ref{subsec: fslr} and introduce the spectral notation used below. Recall that
$
y=f^*(\bx)+\epsilon
$
with $\epsilon\sim\mathcal N(0,\sigma^2)$,
where $\epsilon$ is independent of $\bx$, and $\{(\bx_k,y_k)\}_{k=0}^{N-1}$ are drawn i.i.d.\ from $\cD$. {Let $\bphi:\cX\to\ell^2$ be a feature map, with associated kernel}
\[
{k_{\phi}(\bx,\bx')}={\langle \bphi(\bx),\bphi(\bx')\rangle}.
\]
For $\btheta\in{\ell^2}$, we write
\[
f_{\btheta}(\bx)\coloneqq {\langle\btheta,\bphi(\bx)\rangle}.
\]
The kernel $k_{\phi}$ induces a reproducing kernel Hilbert space. We assume
\[
\EE_{\bx\sim\cD_{\cX}}[{k_{\phi}(\bx,\bx)}]<\infty.
\]

Define the sampling operator ${\mathsf{S}}:{\ell^2}\to L^2(\cD_{\cX})$ by
\[
({\mathsf{S}}\btheta)(\bx)\coloneqq {\langle\btheta,\bphi(\bx)\rangle},
\]
and introduce the covariance operator and kernel integral operator
\[
\cT\coloneqq {\mathsf{S}^*\mathsf{S}}:{\ell^2}\to{\ell^2},
\qquad
\cI\coloneqq {\mathsf{S}\mathsf{S}^*}:L^2(\cD_{\cX})\to L^2(\cD_{\cX}).
\]
Equivalently,
\[
\cT\btheta
=
\EE_{\bx\sim\cD_{\cX}}
\big[
{\langle\btheta,\bphi(\bx)\rangle}\bphi(\bx)
\big],
\]
while
\[
(\cI g)(\bx)
=
\int {k_{\phi}(\bx,\bx')}g(\bx')\,\mathrm d\cD_{\cX}(\bx').
\]
Thus, $\cT$ coincides with the feature covariance operator $\bH$ defined in Section~\ref{subsec: fslr}.

The nonzero eigenvalues of $\cT$ and $\cI$ coincide. Let $\{(\lambda_j,\be_j)\}_{j\ge1}$ be the spectral system of $\cI$, with $\{\be_j\}_{j\ge1}$ orthonormal in $L^2(\cD_{\cX})$, and let $\{\bv_j\}_{j\ge1}\subset{\ell^2}$ be the corresponding eigenvectors of $\cT$, so that
\[
\cI\be_j=\lambda_j\be_j,
\qquad
\cT\bv_j=\lambda_j\bv_j,
\qquad
{\langle\bv_i,\bv_j\rangle}=\delta_{ij}.
\]
The two spectral systems are related by
\begin{equation}
\label{eq:bridge-ej-vj}
\be_j(\bx)
=
\lambda_j^{-1/2}{\langle\bv_j,\bphi(\bx)\rangle},
\end{equation}
where the equality holds in $L^2(\cD_{\cX})$. Consequently,
\[
{k_{\phi}(\bx,\bx')}
=
\sum_{j\ge1}\lambda_j\be_j(\bx)\be_j(\bx'),
\]
with convergence in $L^2(\cD_{\cX}\times\cD_{\cX})$.

To connect this notation with Section~\ref{subsec: fslr}, we express the feature map in the covariance eigenbasis and identify
\[
\phi_j(\bx)\coloneqq\langle\bv_j,\bphi(\bx)\rangle
=\sqrt{\lambda_j}\,\be_j(\bx).
\]

For the source coefficients $(a_j)_{j\ge1}$ in Assumption~\ref{ass: source}, with $\sum_{j\ge1}a_j^2\le1$, define the target coordinates
\[
\theta_j^*
\coloneqq
a_j\lambda_j^{\frac{s-1}{2}},
\qquad j\ge1.
\]
With this identification of $\phi_j$, the source condition can be equivalently written as
\begin{equation}
\label{eq:target-coordinate-expansion}
f^*
=
\sum_{j\ge1}\sqrt{\lambda_j}\,\theta_j^*\be_j
=\sum_{j\ge1}a_j\lambda_j^{s/2}\be_j.
\end{equation}
When $s<1$, the sequence $\{\theta_j^*\}_{j\ge1}$ need not be square-summable, so there need not exist a corresponding target parameter $\btheta^*\in{\ell^2}$. Our analysis below only uses the individual coordinates $\theta_j^*$ and therefore does not require such an ${\ell^2}$-valued representation.

\subsection{Analysis of one-pass SGD}
\label{app:Analysis of SGD dynamics}

We define the population squared risk by
\[
\cR(f)
\coloneqq
\frac12\,\EE_{(\bx,y)\sim\cD}\big[(f(\bx)-y)^2\big].
\]
Since $y=f^*(\bx)+\epsilon$, with $\EE[\epsilon]=0$ and $\EE[\epsilon^2]=\sigma^2$, we have
\[
\cR(f)
=
\frac12\,\EE_{\bx}\big[(f(\bx)-f^*(\bx))^2\big]
+
\frac{\sigma^2}{2}.
\]
The second term is the irreducible noise level. We therefore study the excess risk
\begin{equation}
\label{eq:kernel-excess-risk}
\cE(f)
\coloneqq
\frac12\,\EE_{\bx}\big[(f(\bx)-f^*(\bx))^2\big]
=
\frac12\|f-f^*\|_{L^2(\cD_{\cX})}^2.
\end{equation}

Fix a LR schedule $(\eta_i)_{i=0}^{N-1}$ with $\eta_i\ge0$. Starting from $\btheta_0=0$, one-pass SGD updates the parameter according to
\begin{equation}
\label{eq:sgd-parameter-space}
\btheta_{k+1}
=
\btheta_k
-
\eta_k
\big(
{\langle\btheta_k,\bphi(\bx_k)\rangle}-y_k
\big)
\bphi(\bx_k),
\qquad
k=0,\ldots,N-1.
\end{equation}
For each $j\ge1$, define
\[
\theta_{k,j}
\coloneqq
{\langle\btheta_k,\bv_j\rangle},
\qquad
u_k^j
\coloneqq
\theta_{k,j}-\theta_j^*.
\]
Also define the prediction residual
\[
r_k
\coloneqq
f_{\btheta_k}-f^*.
\]
Using \eqref{eq:bridge-ej-vj} and \eqref{eq:target-coordinate-expansion}, we have the $L^2(\cD_{\cX})$ expansion
\begin{equation}
\label{eq:residual-coordinate-expansion}
r_k(\bx)
=
\sum_{j\ge1}\sqrt{\lambda_j}\,u_k^j\,\be_j(\bx).
\end{equation}
It follows from the orthonormality of $\{\be_j\}_{j\ge1}$ that
\begin{equation}
\label{eq:risk-component-form-app}
\cE(\btheta_k)
\coloneqq
\cE(f_{\btheta_k})
=
\frac12\sum_{j\ge1}\lambda_j(u_k^j)^2.
\end{equation}

We next derive the dynamics of these coordinates. Let
$
\cF_k
\coloneqq
\sigma\big((\bx_0,y_0),\ldots,(\bx_{k-1},y_{k-1})\big)
$
denote the information available before drawing the $k$-th sample. Taking the inner product of \eqref{eq:sgd-parameter-space} with $\bv_j$, and using $y_k=f^*(\bx_k)+\epsilon_k$, gives
\begin{equation}
\label{eq:coord-recursion}
u_{k+1}^j
=
u_k^j
-
\eta_k
\big(
r_k(\bx_k)-\epsilon_k
\big)
{\langle\bv_j,\bphi(\bx_k)\rangle}.
\end{equation}
Conditioned on $\cF_k$, the residual $r_k$ is fixed and $(\bx_k,\epsilon_k)$ is independent of $\cF_k$. Using \eqref{eq:residual-coordinate-expansion} and \eqref{eq:bridge-ej-vj}, we obtain
\begin{align*}
\EE\!\left[
r_k(\bx_k)
{\langle\bv_j,\bphi(\bx_k)\rangle}
\,\middle|\,
\cF_k
\right]
&=
\EE_{\bx}\!\left[
\left(
\sum_{\ell\ge1}\sqrt{\lambda_\ell}\,u_k^\ell\be_\ell(\bx)
\right)
\sqrt{\lambda_j}\be_j(\bx)
\right] =\lambda_j u_k^j.
\end{align*}
Moreover,
\[
\EE\!\left[
\epsilon_k
{\langle\bv_j,\bphi(\bx_k)\rangle}
\,\middle|\,
\cF_k
\right]
=
0.
\]
Define the centered stochastic term
\begin{equation}
\label{eq:def-xi-kj}
\xi_k^j
\coloneqq
\big(
r_k(\bx_k)-\epsilon_k
\big)
{\langle\bv_j,\bphi(\bx_k)\rangle}
-
\lambda_j u_k^j.
\end{equation}
Then
$
\EE[\xi_k^j\mid\cF_k]=0,
$
and \eqref{eq:coord-recursion} becomes
\begin{equation}
\label{eq:coord-recursion-compact}
u_{k+1}^j
=
(1-\eta_k\lambda_j)u_k^j
-
\eta_k\xi_k^j.
\end{equation}

Squaring \eqref{eq:coord-recursion-compact} and conditioning on $\cF_k$, the cross term vanishes because $\EE[\xi_k^j\mid\cF_k]=0$. Hence
\begin{equation}
\label{eq:second-moment-recursion}
\EE\big[(u_{k+1}^j)^2\mid\cF_k\big]
=
(1-\eta_k\lambda_j)^2(u_k^j)^2
+
\eta_k^2\EE\big[(\xi_k^j)^2\mid\cF_k\big].
\end{equation}
Taking expectation gives
\begin{equation}
\label{eq:second-moment-recursion-uncond}
\EE\big[(u_{k+1}^j)^2\big]
=
(1-\eta_k\lambda_j)^2\EE\big[(u_k^j)^2\big]
+
\eta_k^2\EE\big[(\xi_k^j)^2\big].
\end{equation}

\paragraph{Intrinsic time and excess risk bound.}
Define the intrinsic time by $t_k\coloneqq\sum_{i=0}^{k-1}\eta_i$ with $t_0\coloneqq0$. We have the following upper bound for the expected excess risk.
\begin{proposition}[Excess risk bound]
\label{prop:coord-unroll}
{Assume $\eta_{\max}\coloneqq\max_{0\le i\le N-1}\eta_i\le\lambda_1^{-1}$. Then, for every $1\le k\le N$,}
\begin{equation}
\label{eq:risk-volterra-raw}
2\EE[\cE(\btheta_k)]
\le
\sum_{j\ge1}\lambda_j e^{-2\lambda_j t_k}(u_0^j)^2
+
\sum_{i=0}^{k-1}\eta_i^2
\sum_{j\ge1}\lambda_j e^{-2\lambda_j(t_k-t_{i+1})}\,\EE\big[(\xi_i^j)^2\big].
\end{equation}
\end{proposition}
\begin{proof}
For every $j\ge1$, iterating \eqref{eq:second-moment-recursion-uncond} gives
\[
\EE\big[(u_k^j)^2\big]
=
\prod_{i=0}^{k-1}(1-\eta_i\lambda_j)^2(u_0^j)^2
+
\sum_{i=0}^{k-1}\eta_i^2\,\EE\big[(\xi_i^j)^2\big]
\prod_{\ell=i+1}^{k-1}(1-\eta_\ell\lambda_j)^2.
\]
The step-size condition ensures $0\le\eta_\ell\lambda_j\le1$. Hence, for any $0\le i\le k$,
\[
\prod_{\ell=i}^{k-1}(1-\eta_\ell\lambda_j)^2
\le\exp\!\Big(-2\lambda_j\sum_{\ell=i}^{k-1}\eta_\ell\Big)
=e^{-2\lambda_j(t_k-t_i)}.
\]
Recall the risk identity \eqref{eq:risk-component-form-app}. Applying the above estimate  gives \eqref{eq:risk-volterra-raw}.
\end{proof}

The hypercontractivity condition in Assumption~\ref{ass: hypercontra} gives, for all $\bu,\bv\in\ell^2$,
\begin{equation}
\label{eqn:hypercontra}
    \EE_{\bx\sim \cD_{\cX}}\!\big[{\langle \bu, \bphi(\bx)\rangle}^2{\langle \bv,
\bphi(\bx)\rangle}^2\big]
\le
C\cdot
\EE_{\bx}\!\big[{\langle \bu,\bphi(\bx)\rangle}^2\big]\,
\EE_{\bx}\!\big[{\langle \bv,\bphi(\bx)\rangle}^2\big].
\end{equation}

\begin{proposition}[Noise structure]
\label{prop:noise-second-moment}
Suppose \eqref{eqn:hypercontra} holds and $\epsilon_k\sim\mathcal N(0,\sigma^2)$ is independent of $\bx_k$.
{Then, for every $0\le k<N$ and $j\ge1$,}
\[
\EE\big[(\xi_k^j)^2\big]
\le
\lambda_j\big(2C\,\EE[\cE(\btheta_k)]+\sigma^2\big).
\]
\end{proposition}
This proposition shows that the noise variance in each spectral direction scales with the corresponding curvature $\lambda_j$, while its overall amplitude is controlled by the current excess risk together with the label-noise variance~\citep{wu2022alignment, wang2023theoretical}.

\begin{proof}
Recall the definition
\[
\xi_k^j
=
\Big({\langle \bu_k,\bphi(\bx_k)\rangle}-\epsilon_k\Big)\,
{\langle \bv_j,\bphi(\bx_k)\rangle}
-\lambda_j u_k^j.
\]
Let $A_k\coloneqq {\langle \bu_k,\bphi(\bx_k)\rangle}\,{\langle \bv_j,\bphi(\bx_k)\rangle}$ and
$B_k\coloneqq \epsilon_k\,{\langle \bv_j,\bphi(\bx_k)\rangle}$.
Then $\xi_k^j=(A_k-\EE[A_k\mid \bu_k]) - B_k$ since $\EE[A_k\mid \bu_k]=\lambda_j u_k^j$.

Conditioned on $\bu_k$, we have $\EE[B_k\mid \bu_k]=0$ and $\EE[A_k-\EE[A_k\mid \bu_k]\mid \bu_k]=0$.
Moreover, $\epsilon_k$ is independent of $\bx_k$ and $\bu_k$, hence the cross term vanishes:
\[
\EE\!\big[(A_k-\EE[A_k\mid \bu_k])\,B_k\mid \bu_k\big]=0.
\]
Therefore,
\begin{align*}
\EE\big[(\xi_k^j)^2\mid \bu_k\big]
&=
\EE\big[(A_k-\EE[A_k\mid \bu_k])^2\mid \bu_k\big]
+
\EE\big[B_k^2\mid \bu_k\big]\\
&\le
\EE[A_k^2\mid \bu_k]
+
\EE[\epsilon_k^2]\cdot \EE\big[{\langle \bv_j,\bphi(\bx_k)\rangle}^2\big].
\end{align*}
The second term equals $\sigma^2\lambda_j$ because
\[
\EE\big[{\langle \bv_j,\bphi(\bx)\rangle}^2\big]
=
{\langle \bv_j,\cT\bv_j\rangle}
=
\lambda_j.
\]
For the first term, apply \eqref{eqn:hypercontra} with $\bu=\bu_k$ and $\bv=\bv_j$:
\[
\EE[A_k^2\mid \bu_k]
=
\EE\big[{\langle \bu_k,\bphi(\bx_k)\rangle}^2\,{\langle \bv_j,\bphi(\bx_k)\rangle}^2\mid \bu_k\big]
\le
C\cdot
\EE\big[{\langle \bu_k,\bphi(\bx_k)\rangle}^2\mid \bu_k\big]\,
\EE\big[{\langle \bv_j,\bphi(\bx_k)\rangle}^2\big].
\]
Using $\EE[{\langle \bv_j,\bphi(\bx_k)\rangle}^2]=\lambda_j$ and
\[
\EE\big[{\langle \bu_k,\bphi(\bx_k)\rangle}^2\mid \bu_k\big]
=
\EE_{\bx\sim\cD_{\cX}}\big[{\langle \bu_k,\bphi(\bx)\rangle}^2\big]
=
2\,\cE(\btheta_k),
\]
we obtain $\EE[A_k^2\mid \bu_k]\le 2C\,\lambda_j\,\cE(\btheta_k)$.
Combining the two bounds yields
\[
\EE\big[(\xi_k^j)^2 \mid \bu_k\big]
\le
\lambda_j\big(2C\,\cE(\btheta_k)+\sigma^2\big).
\]
Taking expectation over $\bu_k$ gives the desired bound.
\end{proof}

\paragraph{Signal learning and noise forgetting.}
The excess risk bound in Proposition~\ref{prop:coord-unroll}, together with the noise estimate in Proposition~\ref{prop:noise-second-moment}, motivates the following signal-learning function and forgetting kernel:
\begin{equation}
\label{eq:sgd-signal-kernel}
\cS(t)\coloneqq\sum_{j\ge1}\lambda_j(u_0^j)^2e^{-2\lambda_jt},
\qquad
\cK(t)\coloneqq\sum_{j\ge1}\lambda_j^2e^{-2\lambda_jt}.
\end{equation}
These functions put the expected excess-risk bound in the following form.

\begin{proposition}
\label{prop:risk-volterra}
Under the step-size condition of Proposition~\ref{prop:coord-unroll}, assume also that \eqref{eqn:hypercontra} holds. Then, for every $1\le k\le N$,
\begin{equation}
\label{eq:risk-volterra-functions}
2\EE[\cE(\btheta_k)]\le\cS(t_k)
+\sum_{i=0}^{k-1}\eta_i^2(2C\EE[\cE(\btheta_i)]+\sigma^2)\cK(t_k-t_{i+1}).
\end{equation}
\end{proposition}
\begin{proof}
Substitute $\EE[(\xi_i^j)^2]\le\lambda_j(2C\EE[\cE(\btheta_i)]+\sigma^2)$ from Proposition~\ref{prop:noise-second-moment} into \eqref{eq:risk-volterra-raw}, and collect the two sums using \eqref{eq:sgd-signal-kernel}.
\end{proof}

We next establish the properties of $\cS$ and $\cK$ needed to control the excess risk through \eqref{eq:risk-volterra-functions}.

\subsection{Properties of signal-learning function and forgetting kernel}
\label{subsec:bounding}

We establish positivity, monotonicity, power decay, integrability, and a local comparison bound for the functions in \eqref{eq:sgd-signal-kernel}. Based on the source and capacity conditions, we derive power-law upper bounds on their decay.

\begin{proposition}[Signal-learning properties]
\label{prop:bias-term}
Under Assumption~\ref{ass: source} and $\btheta_0=0$, the function $\cS$ is nonnegative and non-increasing. There exists a constant $c_S\eqsim1$, independent of $N$ and the LR schedule, such that, for all $t\ge0$,
\begin{equation}
\label{eq:signal-decay-property}
\cS(t)\lesssim(c_S+t)^{-s}.
\end{equation}
\end{proposition}
\begin{proof}
The source condition gives $u_0^j=-a_j\lambda_j^{(s-1)/2}$, so
\[
\cS(t)=\sum_{j\ge1}a_j^2\lambda_j^s e^{-2\lambda_jt}.
\]
It is easy to verify the nonnegativity and the monotonicity. Since $\sum_j a_j^2\le1$, we have $\cS(t)\le\lambda_1^s$ for all $t\ge0$. For $t\ge1$,
\[
\cS(t)\le\sup_{\lambda\ge0}\lambda^s e^{-2\lambda t}
=\left(\frac{s}{2et}\right)^s.
\]
Take $c_S=1$. Combining the bounds for $0\le t\le1$ and $t\ge1$ gives, for all $t\ge0$,
\[
\cS(t)\le 2^s\max\left\{\lambda_1^s,\left(\frac{s}{2e}\right)^s\right\}(1+t)^{-s}
\lesssim(c_S+t)^{-s},
\]
which proves \eqref{eq:signal-decay-property}.
\end{proof}

\begin{proposition}[Noise-forgetting properties]
\label{lem:spectral-laplace}
Under Assumption~\ref{ass: capacity}, set ${\rho}\coloneqq2-1/q>1$. The function $\cK$ is nonnegative and non-increasing. The following properties hold with $c_K=1$:
\begin{enumerate}
\item[(i)] \textbf{Power-decay upper bound.} For all $t\ge0$,
\begin{equation}
\label{eq:kernel-decay-property}
\cK(t)\lesssim(c_K+t)^{{-\rho}}.
\end{equation}
\item[(ii)] \textbf{Integrability.}
\begin{equation}
\label{eq:kernel-mass-property}
\|\cK\|_1\coloneqq\int_0^\infty\cK(u)\,\dd u
=\tfrac12\tr(\cT)<\infty.
\end{equation}
\item[(iii)] \textbf{Local integral comparison.} For all $r\ge0$ and $0\le h\le(2\lambda_1)^{-1}$,
\begin{equation}
\label{eq:kernel-local-property}
h\cK(r)\le2\int_r^{r+h}\cK(u)\,\dd u.
\end{equation}
\end{enumerate}
\end{proposition}

\begin{proof}
Nonnegativity and monotonicity follow from \eqref{eq:sgd-signal-kernel}. Also,
\[
\cK(0)=\sum_{j\ge1}\lambda_j^2\le\lambda_1\tr(\cT)<\infty.
\]
For $t\ge1$, let $j_0=\lceil t^{1/q}\rceil$. The capacity condition and $\sup_{\lambda\ge0}\lambda^2e^{-2t\lambda}\lesssim t^{-2}$ give
\[
\cK(t)\le\sum_{j\le j_0}\sup_{\lambda\ge0}\lambda^2e^{-2t\lambda}
+\sum_{j>j_0}\lambda_j^2
\lesssim j_0t^{-2}+j_0^{1-2q}
\lesssim t^{-2+1/q}.
\]
Take $c_K=1$. Combining the bound $\cK(t)\le\cK(0)$ for $0\le t\le1$ with the estimate for $t\ge1$ gives $\cK(t)\lesssim(1+t)^{{-\rho}}$ for all $t\ge0$, proving \eqref{eq:kernel-decay-property}. Interchanging the sum and the integral gives
\[
\int_0^\infty\cK(u)\,\dd u
=\sum_{j\ge1}\lambda_j^2\int_0^\infty e^{-2\lambda_ju}\,\dd u
=\frac12\sum_{j\ge1}\lambda_j.
\]
{ This proves \eqref{eq:kernel-mass-property}.} Finally, if $0<h\le(2\lambda_1)^{-1}$, then $2\lambda_jh\le1$ and
\[
\int_r^{r+h}e^{-2\lambda_ju}\,\dd u
=e^{-2\lambda_jr}\frac{1-e^{-2\lambda_jh}}{2\lambda_j}
\ge\frac h2e^{-2\lambda_jr},
\]
where $1-e^{-x}\ge x/2$ for $0\le x\le1$. Multiplying by $\lambda_j^2$ and summing proves \eqref{eq:kernel-local-property}.
\end{proof}

In the following analysis, we write the upper bounds as $\cS(t)\lesssim(1+t)^{-s}$ and $\cK(t)\lesssim(1+t)^{{-\rho}}$. We consider LR schedules satisfying
\begin{equation}
\label{eq:sgd-step-size-threshold}
\eta_{\max}<\frac{1}{2C\,\tr(\cT)}.
\end{equation}

On the one hand, this condition ensures the upper bound on the kernel-weighted sum in \eqref{eq:risk-volterra-functions}, as is shown in the following Lemma. On the other hand, since $\lambda_1\le\tr(\cT)$ and $C\ge1$ for the constant in \eqref{eqn:hypercontra} (by Cauchy--Schwarz inequality), \eqref{eq:sgd-step-size-threshold} also implies
\(
\eta_{\max}<(2\lambda_1)^{-1}.
\)
Thus, the condition in Propositions~\ref{prop:coord-unroll} and~\ref{prop:risk-volterra} holds.

\begin{lemma}[Kernel-weighted sum bound]
\label{lem:kernel-weighted-sum}
Assume the LR schedule satisfies \eqref{eq:sgd-step-size-threshold}. Then, for every $1\le k\le N$,
\begin{equation}
\label{eq:kernel-weighted-sum-bound}
2C\sum_{i=0}^{k-1}\eta_i^2\cK(t_k-t_{i+1})
\le4C\eta_{\max}\|\cK\|_1
<1.
\end{equation}
\end{lemma}
\begin{proof}

For every $1\le k\le N$, the intervals $[t_k-t_{i+1},t_k-t_i]$, with $0\le i<k$, partition $[0,t_k]$. Applying \eqref{eq:kernel-local-property} with $r=t_k-t_{i+1}$ and $h=\eta_i$ gives
\begin{align*}
\sum_{i=0}^{k-1}\eta_i^2\cK(t_k-t_{i+1})
&\le2\sum_{i=0}^{k-1}\eta_i\int_{t_k-t_{i+1}}^{t_k-t_i}\cK(u)\,\dd u\\
&\le2\eta_{\max}\int_0^{t_k}\cK(u)\,\dd u
\le2\eta_{\max}\|\cK\|_1.
\end{align*}
Multiplying by $2C$ proves the first inequality in \eqref{eq:kernel-weighted-sum-bound}. By \eqref{eq:kernel-mass-property} and \eqref{eq:sgd-step-size-threshold},
\[
4C\eta_{\max}\|\cK\|_1
=2C\eta_{\max}\tr(\cT)<1.
\]
This proves \eqref{eq:kernel-weighted-sum-bound}.
\end{proof}

\subsection{Excess risk upper bound}
\label{subsec:risk-from-properties}
We now combine the risk inequality \eqref{eq:risk-volterra-functions} with the properties established in Section~\ref{subsec:bounding}. We first prove uniform boundedness of the excess risk, then estimate the final noise contribution for power-decay schedules.

\begin{lemma}[Uniform boundedness of the excess risk]
\label{prop:risk-const-bound}
For LR schedules satisfying \eqref{eq:sgd-step-size-threshold}, the expected excess risk is uniformly bounded:
\[
\sup_{0\le k\le N}\EE[\cE(\btheta_k)]\lesssim1.
\]
\end{lemma}
\begin{proof}
Let $M_k\coloneqq\max_{0\le i\le k}\EE[\cE(\btheta_i)]$. By Lemma~\ref{lem:kernel-weighted-sum}, for every $1\le k\le N$,
\[
2C\sum_{i=0}^{k-1}\eta_i^2\cK(t_k-t_{i+1})<1.
\]
Combining this bound with \eqref{eq:risk-volterra-functions} and $\cS(t_k)\le\cS(0)$ gives
\[
\EE[\cE(\btheta_k)]
\le\tfrac12\cS(0)+\tfrac12 M_{k-1}+\frac{\sigma^2}{4C}.
\]
Induction therefore yields
\[
M_k\le\max\left\{\EE[\cE(\btheta_0)],\,\cS(0)+\frac{\sigma^2}{2C}\right\}
\lesssim1,
\]
using the bounded initial risk and the signal-decay bound at $t=0$.
\end{proof}

\begin{corollary}[Risk bound in terms of signal learning and noise forgetting]
\label{lem:risk-general-bound}
For LR schedules satisfying \eqref{eq:sgd-step-size-threshold} and every $1\le k\le N$,
\begin{equation}
\label{eq:sgd-functional-risk-bound}
\EE[\cE(\btheta_k)]\lesssim\cS(t_k)
+\underbrace{\sum_{i=0}^{k-1}\eta_i^2\cK(t_k-t_{i+1})}_{\cN_k}.
\end{equation}
\end{corollary}
\begin{proof}
Lemma~\ref{prop:risk-const-bound} bounds $2C\EE[\cE(\btheta_i)]+\sigma^2$ uniformly. Substitution into \eqref{eq:risk-volterra-functions} gives the claim.
\end{proof}

\subsubsection{Power-decay schedules}

We now consider power-decay schedules satisfying \eqref{eq:sgd-step-size-threshold}. For $p\ge0$ and $0\le i<N$, let
\[
\eta_i=\eta_0\Big(1-\frac{i}{N}\Big)^{p},
\]
with $\eta_0>0$. Write $T=t_N$ for the total intrinsic time.

\begin{lemma}[Signal-learning bound for power-decay schedules]
\label{lem:signal-term-power}
For the power-decay schedule above, we have
\[
\cS(T)\lesssim(1+T)^{-s}\eqsim(1+\eta_0N)^{-s}.
\]
\end{lemma}
\begin{proof}
Comparing the power sum with an integral gives
\[
T=\eta_0N^{-p}\sum_{\ell=1}^N\ell^p\eqsim\eta_0N.
\]
The signal-decay property \eqref{eq:signal-decay-property} then gives the claim.
\end{proof}

\begin{lemma}[Noise bound for power-decay schedules]
\label{prop:noise-term-power}
For the power-decay schedule above, let ${a=p/(p+1)}$ and ${\mu=\min\{\rho-1,a\}}$, where ${\rho=2-1/q}$ is the kernel-decay exponent from Proposition~\ref{lem:spectral-laplace}. Then
\begin{equation}
\label{eq:noise-bound-kernel-properties}
\cN_N=\sum_{i=0}^{N-1}\eta_i^2\cK(T-t_{i+1})
\lesssim\eta_0T^{-\mu}
[\log(2+T)]^{\mathbf1_{\{{a=\rho-1}\}}}.
\end{equation}
\end{lemma}
\begin{proof}
For $0<T\le1$, boundedness of $\cK$ gives
\[
\cN_N\lesssim\sum_{i=0}^{N-1}\eta_i^2
\le\eta_0T\le\eta_0T^{-\mu},
\]
which implies the claim also when the logarithmic factor is present.

Now suppose $T>1$. Let $r_i=T-t_{i+1}$ and $m=N-i-1$. Comparing power sums with integrals gives
\[
r_i=\eta_0N^{-p}\sum_{\ell=1}^m\ell^p
\eqsim\eta_0N^{-p}m^{p+1},
\]
where the sum is zero when $m=0$. Using $T\eqsim\eta_0N$, for $p>0$ we obtain
\[
\frac{\eta_i}{\eta_0}=\left(\frac{m+1}{N}\right)^p
\lesssim\left(\frac{r_i}{T}\right)^{{a}}+N^{-p}
\lesssim T^{{-a}}(1+r_i)^{{a}}.
\]
The step-size condition gives $\eta_0\lesssim1$, hence $T\lesssim N$. Together with ${a\le p}$, this gives $N^{-p}\lesssim T^{{-a}}$ and proves the last inequality. Thus,
\begin{equation}
\label{eq:power-lr-envelope}
\eta_i\lesssim\eta_0T^{{-a}}(1+r_i)^{{a}}.
\end{equation}
This also covers the final update $m=0$, and is immediate for $p=0$.

The intervals $[r_i,r_i+\eta_i]$ partition $[0,T]$. By the local comparison and \eqref{eq:power-lr-envelope},
\begin{align}
\cN_N
&\le2\sum_{i=0}^{N-1}\eta_i\int_{r_i}^{r_i+\eta_i}\cK(u)\,\dd u\notag\\
&\lesssim\eta_0T^{{-a}}\sum_{i=0}^{N-1}
\int_{r_i}^{r_i+\eta_i}(1+u)^{{a}}\cK(u)\,\dd u\notag\\
&=\eta_0T^{{-a}}\int_0^T(1+u)^{{a}}\cK(u)\,\dd u.
\label{eq:noise-weighted-kernel-integral}
\end{align}
Here ${a\ge0}$ and $u\ge r_i$ justify the second inequality, including $r_{N-1}=0$. The decay property alone gives
\[
\int_0^T(1+u)^{{a}}\cK(u)\,\dd u
\lesssim\int_0^T(1+u)^{{a-\rho}}\,\dd u
\lesssim
\begin{cases}
1,&{a<\rho-1},\\
\log(2+T),&{a=\rho-1},\\
T^{{a-\rho}+1},&{a>\rho-1}.
\end{cases}
\]
Substitution into \eqref{eq:noise-weighted-kernel-integral} proves \eqref{eq:noise-bound-kernel-properties}.
\end{proof}

\subsection{Proof of Theorem~\ref{thm:power-rate}}
\label{app:proof-theorem-power-cosine}

\begin{proof}
Choose the constant prefactor in $\eta_0$ sufficiently small so that \eqref{eq:sgd-step-size-threshold} holds. Corollary~\ref{lem:risk-general-bound} gives
\[
\EE[\cE(\btheta_N)]\lesssim\cS(T)+\cN_N.
\]
Combining the signal and noise bounds in Lemmas~\ref{lem:signal-term-power} and~\ref{prop:noise-term-power}, and using $T\eqsim\eta_0N$, we obtain
\begin{equation}
\label{eq:power-master-bound}
\EE[\cE(\btheta_N)]
\lesssim(\eta_0N)^{-s}
+\eta_0(\eta_0N)^{-\mu}
[\log(2+\eta_0N)]^{\mathbf1_{\{p=q-1\}}},
\end{equation}
where $\mu=\min\{1-1/q,p/(p+1)\}$. Here the logarithmic factor appears precisely when ${a=\rho-1}$, or equivalently $p=q-1$. We now choose the schedule parameters in the two regimes.

\paragraph{(i) Easy-task regime ($s\ge1-1/q$).}
Choose $p>q-1$. Then $\mu=1-1/q$ and the logarithmic factor is absent. Writing $x\coloneqq\eta_0N$, \eqref{eq:power-master-bound} becomes
\[
\EE[\cE(\btheta_N)]\lesssim x^{-s}+N^{-1}x^{1/q}.
\]
Balancing the two terms gives $x^{s+1/q}\eqsim N$, which corresponds to
\[
\eta_0\eqsim N^{-\frac{sq-q+1}{sq+1}}.
\]
Since $s\ge1-1/q$, the exponent of $N$ in this choice is nonpositive, so a sufficiently small constant prefactor ensures \eqref{eq:sgd-step-size-threshold} for all $N\ge1$. Substitution into the risk bound yields
\[
\EE[\cE(\btheta_N)]\lesssim N^{-\frac{sq}{sq+1}}.
\]

\paragraph{(ii) Hard-task regime ($s<1-1/q$).}
Take $\eta_0\eqsim1$, with a sufficiently small constant to satisfy \eqref{eq:sgd-step-size-threshold}. Then \eqref{eq:power-master-bound} gives
\[
\EE[\cE(\btheta_N)]\lesssim N^{-s}
+N^{-\mu}[\log(2+N)]^{\mathbf1_{\{p=q-1\}}}.
\]
Choose $p>s/(1-s)$, so that $p/(p+1)>s$. Together with $s<1-1/q$, this implies $\mu>s$. Thus, even when the logarithmic factor is present,
\[
N^{-\mu}[\log(2+N)]^{\mathbf1_{\{p=q-1\}}}=o(N^{-s}).
\]
It follows that $\EE[\cE(\btheta_N)]\lesssim N^{-s}$, completing the proof.
\end{proof}

\subsection{Proof of Proposition~\ref{prop:lower-bound-any-lrs}}
\label{app:proof-prop-lower-bound-any-lrs}

\begin{proof}
Let $\eta_{\max}\coloneqq \max_{0\le i\le N-1}\eta_i$ and assume $\eta_{\max}\lesssim1$.
We prove that there exists a data distribution $\cD$ satisfying Assumptions~{\ref{ass: hypercontra},} \ref{ass: capacity} and~\ref{ass: source}
such that $\EE[\cE(\btheta_N)]\gtrsim N^{-s}$. This implies the desired lower bound on
$\sup_{\cD}\EE[\cE(\btheta_N)]$.

Recall from \eqref{eq:coord-recursion-compact} that for each $j\ge 1$,
\[
u_{k+1}^j=(1-\eta_k\lambda_j)u_k^j-\eta_k\xi_k^j,
\qquad \text{with }\EE[\xi_k^j\mid \bu_k]=0.
\]
By the second-moment update \eqref{eq:second-moment-recursion-uncond}, we have
\[
\EE[(u_{k+1}^j)^2]
=
(1-\eta_k\lambda_j)^2\EE[(u_k^j)^2]+\eta_k^2\EE[(\xi_k^j)^2]
\ge (1-\eta_k\lambda_j)^2\EE[(u_k^j)^2].
\]
Iterating this inequality yields, for any $k\ge 1$,
\begin{equation}
\label{eq:lb-second-moment-product}
\EE[(u_k^j)^2]
\ge
\prod_{i=0}^{k-1}(1-\eta_i\lambda_j)^2\,(u_0^j)^2.
\end{equation}

Fix an operator spectrum $\lambda_j=j^{-q}$ (which satisfies the capacity condition with parameter $q$).
Take Gaussian features in $\ell^2$
\[
\bphi(\bx)=\sum_{j\ge1}\sqrt{\lambda_j}\,g_j\bv_j,
\qquad g_j\stackrel{\mathrm{iid}}{\sim}\mathcal N(0,1).
\]
Since $\sum_j\lambda_j<\infty$, these features lie in ${\ell^2}$ almost surely, with covariance $\cT\bv_j=\lambda_j\bv_j$ and kernel eigenfunctions $\be_j(\bx)=g_j$.
For any feature projections $X={\langle\bu,\bphi(\bx)\rangle}$ and $Y={\langle\bv,\bphi(\bx)\rangle}$, we have the Gaussian identity
\[
\EE[X^2Y^2]=\EE[X^2]\EE[Y^2]+2\big(\EE[XY]\big)^2
\]
By Cauchy--Schwarz inequality, we verify Assumption~\ref{ass: hypercontra} with constants $1$ and $3$, uniformly in $N$ and the target.

Choose a target function supported on a single eigen-direction:
\[
f^*(\bx)=\lambda_{j^*}^{s/2}\,\be_{j^*}(\bx),
\quad\text{i.e.,}\quad
a_{j^*}=1,\ a_j=0\ (j\neq j^*),
\]
where the index $j^*$ is to be determined. This target function satisfies the source condition (Assumption~\ref{ass: source}) with $\sum_j a_j^2=1$.
{We take independent regression noise $\epsilon\sim\mathcal N(0,1)$.}
Note that the lower bound below only uses \eqref{eq:lb-second-moment-product}, and therefore it does not depend on the noise distribution.

We initialize $\btheta_0=0$, so $u_0^j=-\theta_j^*$.
For the above $f^*$, the corresponding coefficients in ${\ell^2}$ satisfy
$\theta_{j^*}^*=\lambda_{j^*}^{\frac{s-1}{2}}$ and $\theta_j^*=0$ for $j\neq j^*$,
hence
\begin{equation}
\label{eq:u0-single}
(u_0^j)^2=
\begin{cases}
\lambda_{j^*}^{s-1}, & j=j^*,\\
0, & j\neq j^*.
\end{cases}
\end{equation}
Recall $\cE(\btheta_k)=\frac12\sum_{j\ge1}\lambda_j(u_k^j)^2$ (cf.\ \eqref{eq:risk-component-form-app}).
Combining \eqref{eq:lb-second-moment-product} and \eqref{eq:u0-single} gives
\[
\EE[\cE(\btheta_N)]
\ge
\frac12\,\lambda_{j^*}\,\EE[(u_N^{j^*})^2]
\ge
\frac12\,\lambda_{j^*}\prod_{i=0}^{N-1}(1-\eta_i\lambda_{j^*})^2\lambda_{j^*}^{s-1}
=
\frac12\,\lambda_{j^*}^{s}\prod_{i=0}^{N-1}(1-\eta_i\lambda_{j^*})^2.
\]
Let $T\coloneqq t_N=\sum_{i=0}^{N-1}\eta_i$. Since $T\le \eta_{\max}N\lesssim N$, it suffices to show
{$\EE[\cE(\btheta_N)]\gtrsim(1+T)^{-s}$.}
Define
\[
j_0\coloneqq \min\Big\{j\ge 1:\ \eta_{\max}\lambda_j\le \tfrac12\Big\}.
\]
For any $j\ge j_0$, we have $\eta_i\lambda_j\le 1/2$ for all $i$, and thus
$1-x\ge e^{-2x}$ for $x\in[0,1/2]$. This gives
\[
\prod_{i=0}^{N-1}(1-\eta_i\lambda_j)^2
\ge
\prod_{i=0}^{N-1} e^{-4\eta_i\lambda_j}
=
e^{-4\lambda_j T}.
\]
Hence for any $j^*\ge j_0$,
\begin{equation}
\label{eq:lb-risk-exp}
\EE[\cE(\btheta_N)]
\ge
\frac12\,\lambda_{j^*}^{s}\,e^{-4\lambda_{j^*}T}.
\end{equation}

Let
\[
j^*\coloneqq\max\left\{j_0,\left\lceil(1+T)^{1/q}\right\rceil\right\}.
\]
Since $\eta_{\max}\lesssim1$, the index $j_0$ is uniformly bounded independently of $N$ and the schedule. Hence
$\lambda_{j^*}=(j^*)^{-q}\eqsim(1+T)^{-1}$ and $\lambda_{j^*}T\le1$.
Therefore, for all $T\ge0$, \eqref{eq:lb-risk-exp} gives that
\[
\EE[\cE(\btheta_N)]
\ge\tfrac12\lambda_{j^*}^s e^{-4\lambda_{j^*}T}
\gtrsim(1+T)^{-s}.
\]
Since $T\le\eta_{\max}N\lesssim N$, this proves
\[
\sup_{\cD}\EE[\cE(\btheta_N)]\gtrsim N^{-s}.
\]

\end{proof}

\section{Additional experiments}
\label{app:additional-experiments}

\subsection{PLK regression}
\label{app:additional-plk}

We provide additional numerical experiments to test Theorem~\ref{thm: optimal-fractional-lrs} beyond the representative cosine and power decay in the main text. We consider three fractional schedules, cosine decay, linear decay, and 1-sqrt decay, and choose representative $(s,q)$ settings covering the three regimes in Figure~\ref{fig: intro}(right).

For each schedule and $(s,q)$ setting, we run one-pass SGD in PLK regression over different training horizons $N$ and tune the peak LR to minimize the average final-step normalized loss. The dashed lines in Figure~\ref{fig:appendix_4learning_rates_column} show the corresponding theoretical scaling rates predicted by Theorem~\ref{thm: optimal-fractional-lrs}.

Across all tested schedules and $(s,q)$ settings, the observed log-log slopes agree well with the theoretical predictions. In particular, the empirical scaling matches the predicted rate in each of the blue, red, and green regimes of Figure~\ref{fig: intro}(right). These experiments therefore support the full phase diagram and show that schedule-induced capacity saturation is not specific to the cosine and power-decay comparison in the main text.

\begin{figure}[!h]
\centering
    \includegraphics[width=\textwidth]{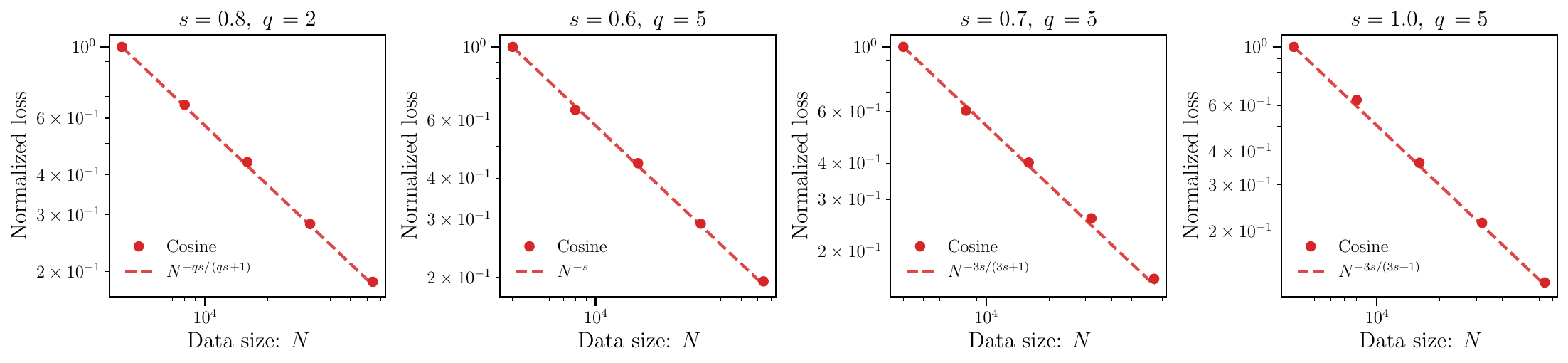}\\[-0.5em]
    {\small \textbf{(a)} Cosine decay.}

    \vspace*{.5em}
    \includegraphics[width=\textwidth]{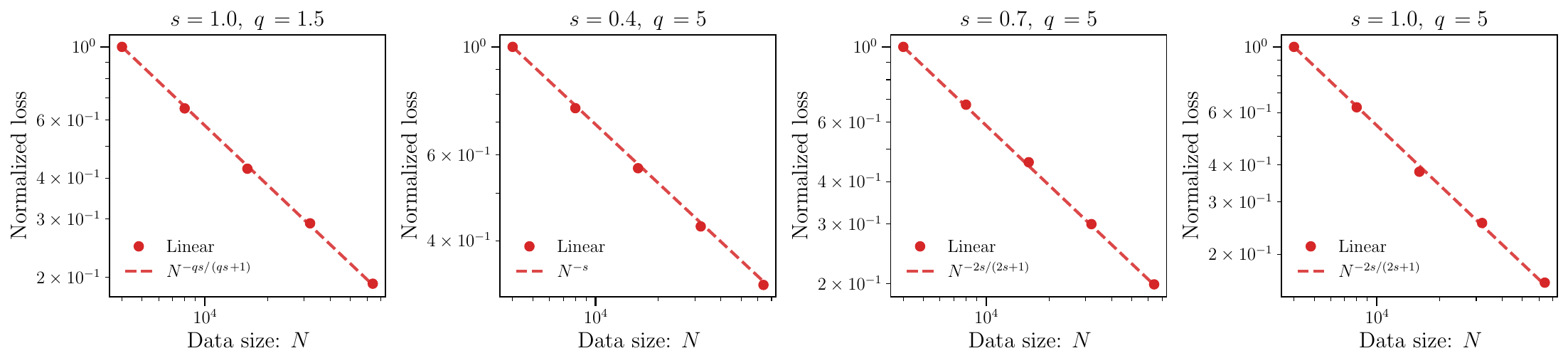}\\[-0.5em]
    {\small \textbf{(b)} Linear decay.}\\[-0.2em]

    \vspace*{.5em}
    \includegraphics[width=\textwidth]{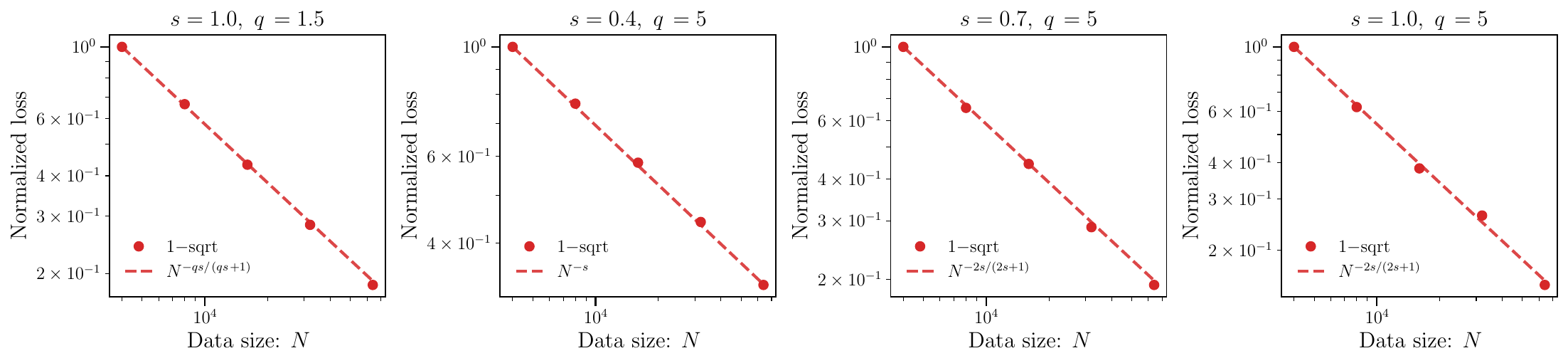}\\[-0.5em]
    {\small \textbf{(c)} 1-sqrt decay.}\\[-0.2em]

\caption{
\textbf{Scaling rates of fractional LR schedules across the regimes in Theorem~\ref{thm: optimal-fractional-lrs}.}
Each row corresponds to one schedule family, while the four columns use representative settings from the blue, red, and green regions of Figure~\ref{fig: intro}(right). The first two columns illustrate the optimal easy- and hard-task rates, whereas the last two show the capacity-saturated rate governed by $\alpha=\min\{q,p+1\}$. Dashed lines indicate the theoretical predictions. For each $N$, the peak LR is tuned separately, and the reported loss is averaged over $500$ independent SGD runs.
}
\label{fig:appendix_4learning_rates_column}
\end{figure}

\subsection{MLP regression}
\label{app:additional-mlp-lr005}

We provide additional Fourier regression experiments using the setup of Section~\ref{subsec:mlp-fourier}, changing only the peak LR to $0.05$. As shown in Figure~\ref{fig:appendix-mlp-lr005}, the optimal decay ratio for the easier target $\gamma=10$ is now $0.9$, but its validation loss is nearly identical to that at $r=1$.

\begin{figure}[!htbp]
\centering
    \begin{minipage}[b]{0.43\textwidth}
        \centering
        \includegraphics[width=\linewidth]{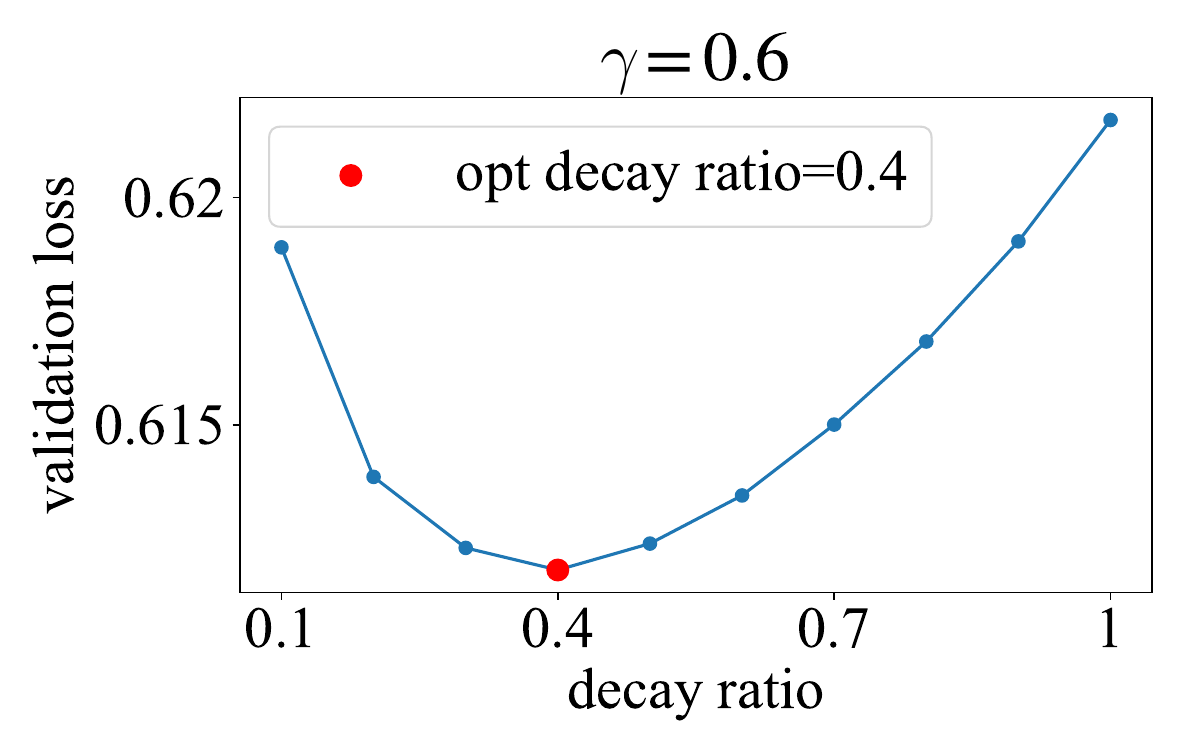}
    \end{minipage}
    \hspace*{2em}
    \begin{minipage}[b]{0.43\textwidth}
        \centering
        \includegraphics[width=\linewidth]{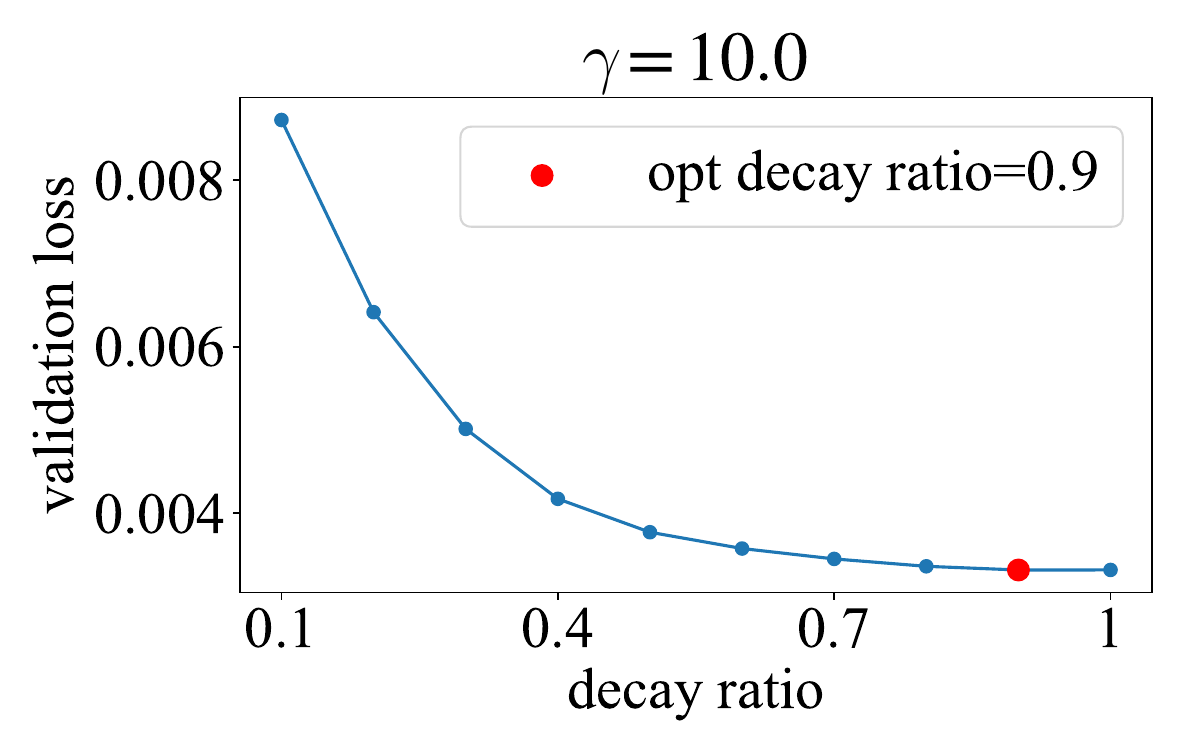}
    \end{minipage}
\caption{\small
\textbf{Decay ratio versus validation loss for two-layer MLPs with peak LR $0.05$.} All other settings match the main-text Fourier experiment. \textbf{(left)} For $\gamma=0.6$, the mean loss is minimized at $r=0.4$. \textbf{(right)} For $\gamma=10$, the mean loss is minimized at $r=0.9$, with nearly identical loss at $r=1$. Points show means over $32$ runs, with red markers indicating the grid minimizers.}
\label{fig:appendix-mlp-lr005}
\end{figure}

\subsection{CIFAR-100 classification}
\label{app:additional-cifar100}

We provide additional CIFAR-Full and CIFAR-Binary experiments using the setup of Section~\ref{subsec:cifar100}, changing only the peak LR to $0.1$ (Figure~\ref{fig:appendix-cifar-lr01}) or the batch size to $50$ (Figure~\ref{fig:appendix-cifar-b50}).

\begin{figure}[!htbp]
\centering
    \begin{minipage}[b]{0.43\textwidth}
        \centering
        \includegraphics[width=\linewidth]{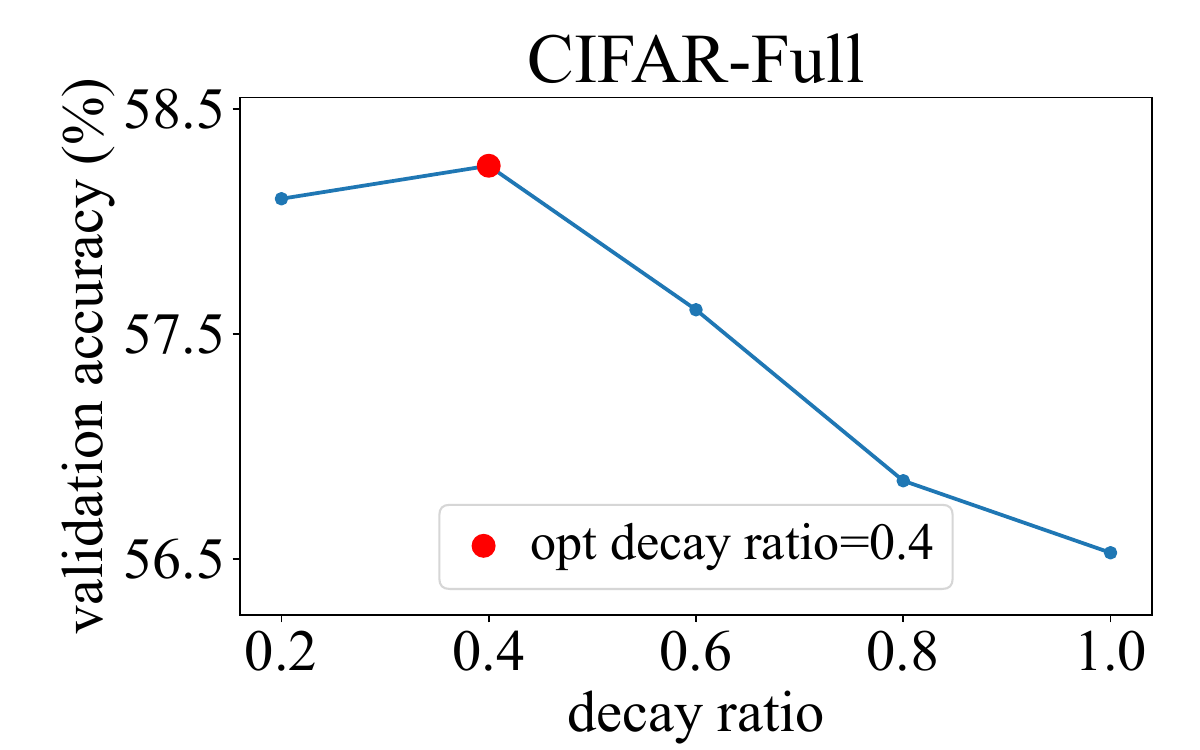}
    \end{minipage}
    \hspace*{2em}
    \begin{minipage}[b]{0.43\textwidth}
        \centering
        \includegraphics[width=\linewidth]{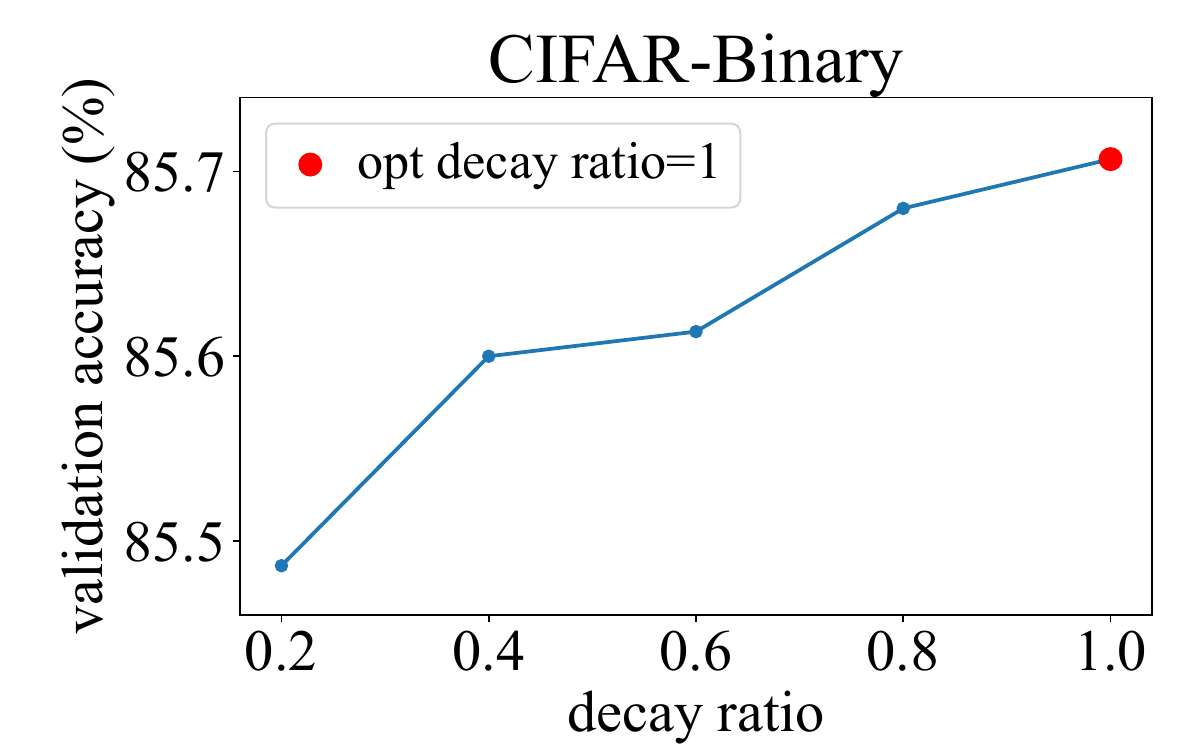}
    \end{minipage}
\caption{\small
\textbf{Decay ratio versus final validation accuracy with peak LR $0.1$.} Both tasks use batch size $20$ and $5$ training epochs. \textbf{(left)} CIFAR-Full is maximized at $r=0.4$; \textbf{(right)} CIFAR-Binary is maximized at $r=1$. Points show means over three runs, with red markers indicating the grid maximizers of the mean curves.}
\label{fig:appendix-cifar-lr01}
\end{figure}

\begin{figure}[!htbp]
\centering
    \begin{minipage}[b]{0.43\textwidth}
        \centering
        \includegraphics[width=\linewidth]{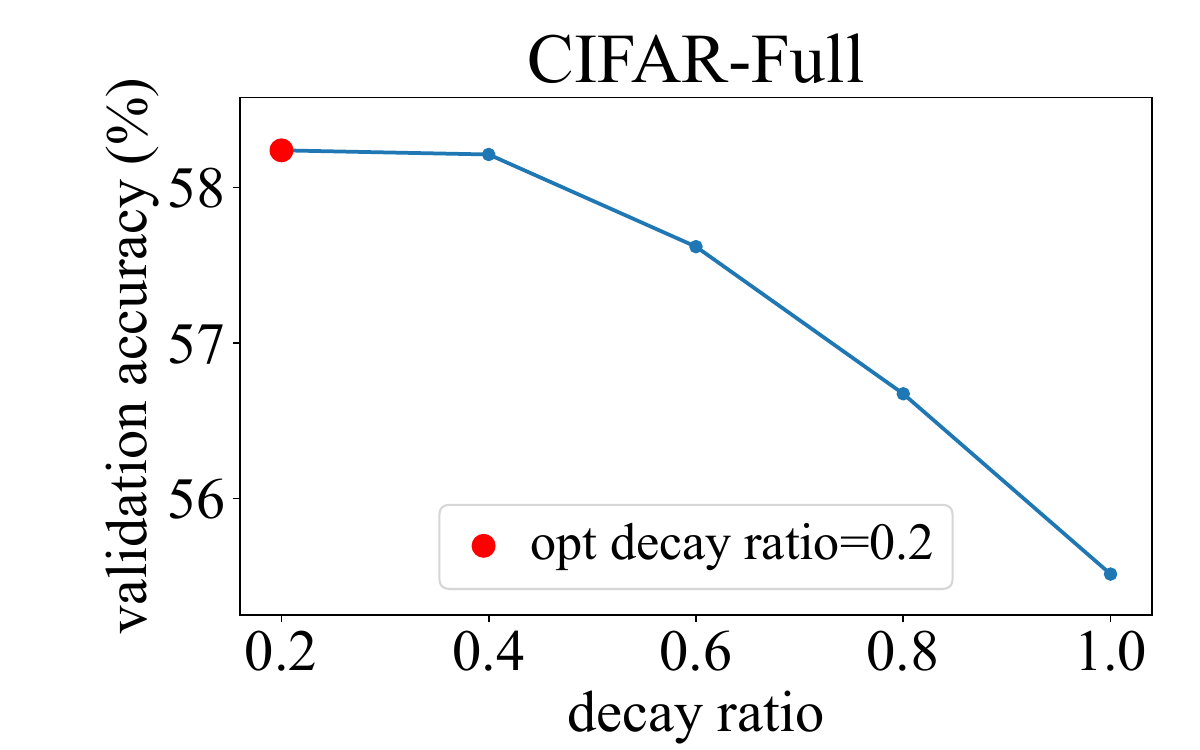}
    \end{minipage}
    \hspace*{2em}
    \begin{minipage}[b]{0.43\textwidth}
        \centering
        \includegraphics[width=\linewidth]{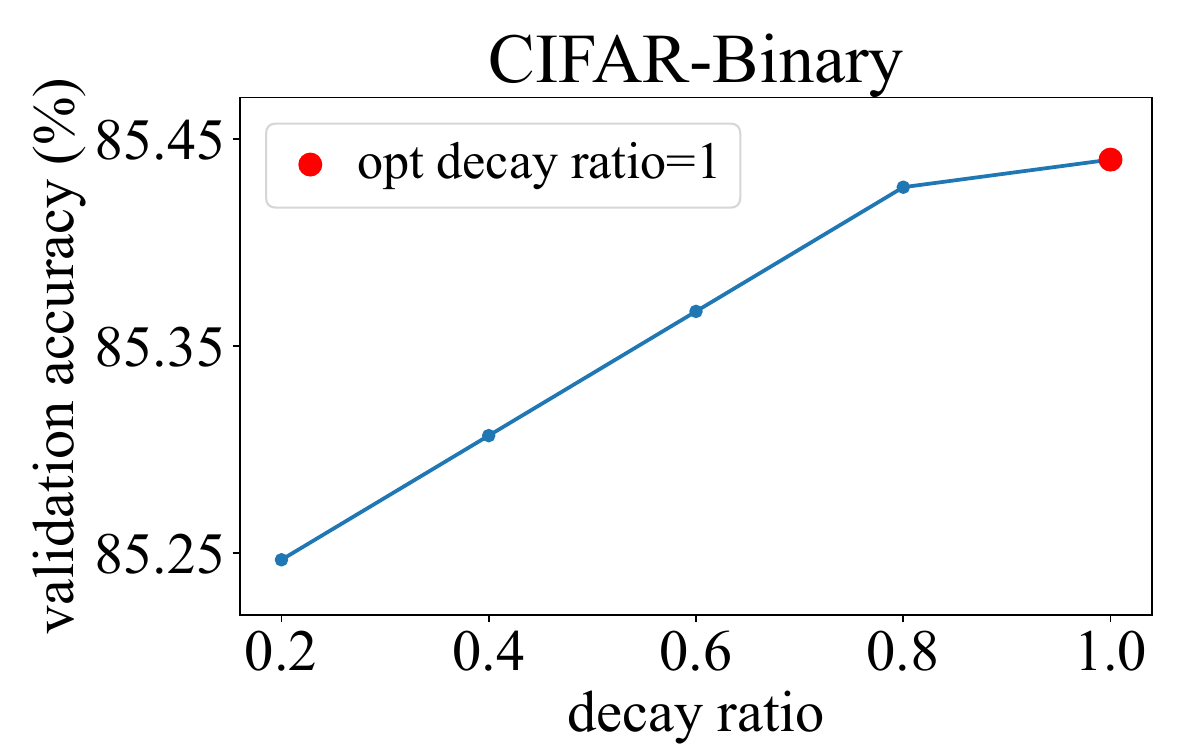}
    \end{minipage}
\caption{\small
\textbf{Decay ratio versus final validation accuracy with batch size $50$.} Both tasks use peak LR $0.2$ and $5$ training epochs. \textbf{(left)} CIFAR-Full is maximized at $r=0.2$; \textbf{(right)} CIFAR-Binary is maximized at $r=1$. Points show means over three runs, with red markers indicating the grid maximizers of the mean curves.}
\label{fig:appendix-cifar-b50}
\end{figure}

\end{document}